 \documentclass[final,5p,times,twocolumn]{elsarticle}

\usepackage{epsfig}

\usepackage{amsmath}
\usepackage{amssymb}
\usepackage{mathtools}
\usepackage{amsthm}
\usepackage{bbm}
\usepackage{float}
\usepackage[capitalize,noabbrev]{cleveref}

\usepackage{microtype}
\usepackage{graphicx}
\usepackage{subfigure}
\usepackage{booktabs} 
\usepackage{multirow}
\usepackage{amsfonts}
\usepackage{amsmath}
\usepackage{amsthm}
\usepackage{tabto}
\usepackage{url}

\usepackage{thmtools, thm-restate}

\DeclareMathOperator*{\argmin}{arg\,min}

\theoremstyle{plain}
\newtheorem{theorem}{Theorem}[section]

\newtheorem{lemma}[theorem]{Lemma}
\newtheorem{corollary}[theorem]{Corollary}
\theoremstyle{definition}
\newtheorem{definition}[theorem]{Definition}
\newtheorem{assumption}[theorem]{Assumption}
\theoremstyle{remark}
\newtheorem{remark}[theorem]{Remark}

\usepackage[textsize=tiny]{todonotes}

\usepackage[numbers]{natbib}

\journal{Neurocomputing}

\begin{document}
\begin{frontmatter}



\title{Enhanced Physics-Informed Neural Networks with Augmented Lagrangian Relaxation Method (AL-PINNs)}


\author[inst1]{Hwijae Son\corref{cor1}}
\ead{hjson@hanbat.ac.kr}
\author[inst2]{Sung Woong Cho\corref{cor1}}
\ead{swcho95kr@kaist.ac.kr}
\author[inst3]{Hyung Ju Hwang \corref{cor2}}
\ead{hjhwang@postech.ac.kr}

\cortext[cor1]{These authors contributed equally to this work.}
\cortext[cor2]{Corresponding author.}
\cortext[cor3]{Source code is available at: https://github.com/HwijaeSon/AL-PINNs}

\affiliation[inst1]{organization={Department of Artificial Intelligence Software},
            addressline={Hanbat National University}, 
            city={Daejeon},
            postcode={34158}, 
            country={Republic of Korea}}
            
\affiliation[inst2]{organization={Stochastic Analysis and Application Research Center},
            addressline={Korea Advanced Institute of Science and Technology}, 
            city={Daejeon},
            postcode={34141}, 
            country={Republic of Korea}}

\affiliation[inst3]{organization={Department of mathematics},
            addressline={Pohang University of Science and Technology}, 
            city={Pohang},
            postcode={37673}, 
            country={Republic of Korea}}

\begin{abstract}
Physics-Informed Neural Networks (PINNs) have become a prominent application of deep learning in scientific computation, as they are powerful approximators of solutions to nonlinear partial differential equations (PDEs). There have been numerous attempts to facilitate the training process of PINNs by adjusting the weight of each component of the loss function, called adaptive loss-balancing algorithms. In this paper, we propose an Augmented Lagrangian relaxation method for PINNs (AL-PINNs). We treat the initial and boundary conditions as constraints for the optimization problem of the PDE residual. By employing Augmented Lagrangian relaxation, the constrained optimization problem becomes a sequential max-min problem so that the learnable parameters $\lambda$ adaptively balance each loss component. Our theoretical analysis reveals that the sequence of minimizers of the proposed loss functions converges to an actual solution for the Helmholtz, viscous Burgers, and Klein--Gordon equations. We demonstrate through various numerical experiments that AL-PINNs yield a much smaller relative error compared with that of state-of-the-art adaptive loss-balancing algorithms. 
\end{abstract}




\end{frontmatter}



\section{Introduction}\label{sec1}
Starting from a seminal work, Physics-Informed Neural Networks (PINNs) \cite{raissi2019physics} have become a significant research interest in many scientific disciplines, along with the great development of deep learning. Due to its simple and easy-to-implement algorithm and powerful approximation capacity, numerous successful applications of PINNs have been reported in the last decade \cite{lu2021deepxde}. We refer readers to a recent review by \cite{karniadakis2021physics} for more information.

There are several branches of theoretical convergence results for PINNs. For example, the neural network converges to a classical solution for the linear second-order elliptic and parabolic equations \cite{shin2020convergence}. \citet{jagtap2022deep} proposed a neural network incorporating an adaptive activation function and conducted an analysis of its gradient flow dynamics to demonstrate faster convergence. Another branch of studies analyzed continuous loss functions and proved that an actual error can be bounded by a continuous function of each loss component \cite{sirignano2018dgm, jo2020deep,hwang2020trend} (e.g., \eqref{pinn_loss}). However, the precise functional form of such an upper bound, as well as the best loss function to approximate solutions of given PDEs, remain to be discovered.

Recently, there has been a considerable effort to find the best surrogate loss function by manipulating the ratio of each loss component to form a loss function, using loss-balancing algorithms (see Section \ref{sec_related}). However, most current approaches are limited to individual empirical observations, such as imbalanced gradients, stiffness in the solution, and discrepancy in the convergence, which may have a detrimental effect on training PINNs. Our first motivation arises from this point: We need a universal loss-balancing algorithm that we can apply without any prior observations or knowledge. Our second motivation comes from doubt for a common belief on a large penalty parameter is enough for the constrained optimization for PINNs. Several previous works employed a large multiplicative penalty parameter $\beta$ which is predefined before training \cite{yu2017deep, chen2020physics}. However, to the best of our knowledge, this heuristic has not been theoretically justified. In this paper, we propose a universal approach by setting the initial and boundary conditions as constraints for the optimization problem. Then we reformulate the constrained optimization problem to an unconstrained one using the augmented Lagrangian method. Furthermore, we provide a rigorous proof that demonstrates how the proposed method generates a sequence of neural networks that converge to the true solution.

The augmented Lagrangian relaxation method has been widely applied in the field of constrained deep learning, with successful results reported in the literature (see Section \ref{sec_related}). For the problems involving PDEs, the constrained optimization methods are intensively studied using the Deep Ritz Method (DRM) due to their constrained nature. For example, the convergence of the penalty method for the DRM is given by \cite{muller2019deep}, and a deep augmented Lagrangian method for the DRM is proposed by \cite{huang2021augmented}. However, the convergence of the augmented Lagrangian method has never been discovered for either PINNs or the DRM. To the best of our knowledge, this is the first attempt to show the convergence of the augmented Lagrangian method for PINNs.

In this paper, we propose the Augmented Lagrangian relaxation method for training PINNs (AL-PINNs) to facilitate the training of PINNs. Considering the initial and boundary conditions as constraints, we reformulate the training of PINNs into a constrained optimization problem. Using the augmented Lagrangian relaxation method, we derive a novel sequence of loss functions with adaptively balanced loss components. In Section \ref{sec3}, we prove that the minimizers of the loss functions converge to an actual solution. In Section \ref{sec4}, we first detail experiments exhibiting the advantages of the augmented Lagrangian relaxation compared with the penalty, and Lagrange multiplier methods. We then provide experimental results that demonstrate the outstanding performance of the proposed AL-PINNs compared with several adaptive loss-balancing algorithms using the Helmholtz, viscous Burgers, and Klein--Gordon equations. Therefore, the proposed AL-PINNs embody a convergence-guaranteed universal framework that consistently outperforms other loss-balancing algorithms in the solution of PDEs.

\subsection{Related works}\label{sec_related}
\textbf{Constrained Deep Learning.}
Imposing hard constraints on the output of an artificial neural network is a challenging problem. \citet{marquez2017imposing} discussed the possibility of imposing hard constraints on the output of a neural network in a computationally feasible way by using the Krylov subspace method. However, they also acknowledged that the performance of the proposed method is not superior to that of the soft-constrained approach. On the other hand, from a soft constraint perspective, the Augmented Lagrangian method (ALM), or equivalently, Lagrangian dual formulation, has been widely adopted for solving constrained optimization problems involving neural networks. For instance, \citet{nandwani2019primal} demonstrated that a constrained formulation with ALM yields state-of-the-art performance in three NLP benchmarks. \citet{sangalli2021constrained} presented the use of ALM for solving class-imbalanced binary classification, and  \citet{fioretto2020lagrangian} applied ALM to optimal power flow prediction problems. For problems involving PDEs, \citet{hwang2021lagrangian} proposed an ALM approach to impose several physical conservation laws of kinetic PDEs on the neural network, and \citet{lu2021physics} proposed PINNs with hard constraints for the inverse design. For an extrapolation problem, \citet{kim2021dpm} proposed the Dynamic Pooling Method (DPM) to impose a soft constraint on a residual loss function for training PINNs. 

\textbf{Imposing the initial and boundary conditions.} 
The use of initial and boundary conditions as hard constraints is frequently considered in PINNs literature. Several studies proposed to set the boundary conditions as hard constraints by utilizing a distance function $dist(x, \partial\Omega)$ (see \cite{lagaris1998artificial, berg2018unified, jo2020deep, son2021sobolev, sukumar2021exact, schiassi2021extreme} for examples). However, in most existing studies regarding PDEs and neural networks, the boundary conditions are relaxed into the loss function in a soft manner using a quadratic penalty function (see \cite{sirignano2018dgm, raissi2019physics, yu2017deep} for examples). For the Deep Ritz Method (DRM), \citet{muller2019deep} has shown that the sequence of quasi-minimizers for the variational problem with the penalty method converges to a true solution. 

\textbf{loss-balancing algorithms for PINNs.} 
loss-balancing algorithms have been widely studied to deal with various kinds of stability issues in the training dynamics of PINNs. For instance, a non-adaptive weighting strategy that considers the weights as hyperparameters is proposed by \cite{zhao2020solving}. \cite{mcclenny2020self} considered the use of a soft attention mechanism to adaptively balance the components of the loss function, giving more weight to regions where the solution exhibits a stiff transition. \citet{xiang2022self} evaluated the likelihood of the observed data using a neural network and proposed a method that maximizes the likelihood by varying the weights of the loss functions. \citet{wang2021understanding} argued that the numerical stiffness in gradient statistics leads to unstable back-propagation and introduced a solution in the form of an adaptive loss-balancing algorithm known as learning rate annealing. \citet{wang2022and} observed a discrepancy in the convergence rate of loss components, and proposed to use the eigenvalues of the NTK to balance the convergence rate. Another branch of study considers the training of PINNs as a multi-objective learning problem, in which the individual components compete with each other. (see, \cite{van2020optimally, bischof2021multi, rohrhofer2021pareto} for more information). 

\section{Preliminaries and Methods}\label{sec2}
\subsection{Preliminaries}
Consider a generic constrained optimization problem on $\mathbb{R}^n$:\begin{equation}\label{general_constrained}
    \begin{aligned}
        &\argmin_{\theta} \mathcal{J}(\theta), \\
        &\text{subject to } C(\theta)=0, \theta\in\mathbb{R}^n, C:\mathbb{R}^n\to\mathbb{R}^m.
    \end{aligned}
\end{equation}
Constrained optimization problems have been deeply investigated in convex optimization literature \cite{boyd2004convex}. A naive approach to solve \eqref{general_constrained} is to relax the constraints into the objective function via the penalty method, i.e.,
\begin{equation}
    \mathcal{J}_n(\theta) = \mathcal{J}(\theta) + \beta_n \|C(\theta)\|_2^2, \nonumber
\end{equation}
where $\beta_n \to \infty$. However, this approach exhibits numerical instabilities due to large values of $\beta_n$ \cite{bertsekas1976multiplier}.

Another method for solving \eqref{general_constrained} is to consider the Lagrangian duality
\begin{equation}
    \mathcal{J}_\lambda(\theta) = \mathcal{J}(\theta) + \langle\lambda, C(\theta)\rangle_{\mathbb{R}^m}, \nonumber
\end{equation}
where $\lambda\in\mathbb{R}^m$ and $\langle\cdot,\cdot\rangle_{\mathbb{R}^m}$ denotes the standard inner product on $\mathbb{R}^m$. Since \begin{equation}
    \min_{C(\theta)=0} \mathcal{J(\theta)} = \min_{C(\theta)=0} \mathcal{J}_{\lambda}(\theta) \geq \min_{\theta \in \mathbb{R}^n} \mathcal{J}_{\lambda}(\theta), \nonumber
\end{equation} one can reformulate \eqref{general_constrained} into \begin{equation}
    \max_{\lambda}\min_{\theta\in\mathbb{R}^n} \mathcal{J}_{\lambda}(\theta). \nonumber
\end{equation}
However, this approach is only valid when the original problem has a locally convex structure \cite{bertsekas1976multiplier}. 

The augmented Lagrangian method combines the above two approaches. A new objective function for solving \eqref{general_constrained} via augmented Lagrangian relaxation of the constraints reads as follows:
\begin{equation}
    \mathcal{J}_{\beta, \lambda}(\theta) = \mathcal{J}(\theta) + \beta\|C(\theta)\|_2^2 + \langle\lambda, C(\theta)\rangle_{\mathbb{R}^m}. \nonumber
\end{equation}
This approach avoids the instability that results from the large penalty parameter, and the locally convex structure is not required \cite{bertsekas1976multiplier}. Again, one can reformulate \eqref{general_constrained} into a max-min problem :
\begin{equation}
    \max_{\lambda}\min_{\theta\in\mathbb{R}^n} \mathcal{J}_{\lambda}(\theta). \nonumber
\end{equation}
Because the objective function used in cases involving neural networks is highly non-convex, one cannot directly apply the convergence theory developed for convex optimization settings. Nevertheless, recent studies have found that the augmented Lagrangian method performs well in various kinds of constrained optimization problems involving neural networks \cite{nandwani2019primal, fioretto2020lagrangian, sangalli2021constrained, hwang2021lagrangian, basir2022physics}. In this study, we apply the augmented Lagrangian method to the constrained formulation of PINNs.

\subsection{Augmented Lagrangian Methods for PINNs}
We consider a general class of PDEs that reads as:
\begin{equation}\label{general_pde}
    \begin{aligned}
        &Nu = f, &&\text{ for } x\in\Omega, \\
        &Tu = g, &&\text{ for } x\in\partial\Omega,
    \end{aligned}
\end{equation} 
where $N$ and $T$ denote the differential operator and trace operator, respectively. In the original formulation of PINNs, we minimize a loss function to penalize a neural network to satisfy \eqref{general_pde}: 
\begin{equation}\label{pinn_loss}
    \mathcal{L}(\theta) = \|Nu_{nn}(\theta) - f\|_{L^2(\Omega)} + \|Tu_{nn}(\theta) - g\|_{L^2(\partial\Omega)},
\end{equation}
where $u_{nn}$ denotes the neural network solution and $\theta$ denotes a set of parameters. However, as the penalty method often fails to approximate an accurate solution, adaptive loss-balancing algorithms are commonly applied in the literature. As pointed out in previous works, the boundary conditions often cause instability in the training of PINNs and DRM (See, \cite{wang2021understanding, wang2022and,muller2021notes, muller2019deep}).

In this study, we propose a novel class of loss functions for PINNs based on the augmented Lagrangian method, which aim to solve the following constrained optimization problem:\begin{equation}\label{constrained_optimization}
    \begin{aligned}
        &\argmin_{\theta} \|Nu_{nn}(\theta) - f\|_{L^2(\Omega)}, \\
        &\text{subject to } Tu_{nn}(\theta) = g.
    \end{aligned}
\end{equation}
In the augmented Lagrangian method, the constraint in \eqref{constrained_optimization} is relaxed into the objective function via the Lagrange multiplier $\lambda \in L^2(\partial\Omega)$. The resulting objective function reads as:
\begin{equation}\label{loss_function}
\begin{aligned}
    \mathcal{L}_{\lambda}(\theta) &= \|Nu_{nn}(\theta) - f\|_{L^2(\Omega)}  \\&+\beta\|Tu_{nn}(\theta)-g\|_{L^2(\partial\Omega)} 
    \\&+\langle \lambda, Tu_{nn}(\theta) - g\rangle_{L^2(\partial\Omega)} 
\end{aligned}
\end{equation}
where $\langle\cdot,\cdot\rangle_{L^2(\partial\Omega)}$ is the standard $L^2$ inner product.

As discussed in the previous subsection, we solve the following max-min problem:
\begin{equation}\label{maxmin}
    \max_{\lambda\in L^2(\partial\Omega)} \min_{\theta} \mathcal{L}_{\lambda}(\theta).
\end{equation}
We solve the above max-min problem by using the gradient descent-ascent algorithm for $\theta$-$\lambda$. The update rules are given as:
\begin{equation}
\begin{aligned}\nonumber
    &\theta \leftarrow \theta - \eta_{\theta} \nabla_{\theta} \mathcal{L}_{\lambda}(\theta), \\
    &\lambda \leftarrow \lambda + \eta_{\lambda} \nabla_{\lambda} \mathcal{L}_{\lambda}(\theta),
\end{aligned}
\end{equation}where $\eta_{\theta}$ and $\eta_{\lambda}$ are the predefined learning rates for $\theta$ and $\lambda$, respectively.

In the numerical experiments, we treat $\lambda$ as a discretization, i.e., $\lambda\approx(\lambda(x^b_1),...,\lambda(x^b_{N_b}))$. We discretize the loss function in \eqref{loss_function} on uniform grid points. Let $\{x^r_1, x^r_2, ..., x^r_{N_r}\} \subset \Omega$, and $\{x^b_1, x^b_2, ..., x^b_{N_b}\} \subset \partial\Omega$ be the uniform grid points. Then, the objective function is discretized into 
\begin{align*}
        \mathcal{L}_{\lambda}(\theta) &\approx \frac{|\Omega|}{N_r} \sum_{i=1}^{N_r} (Nu_{nn}(x^r_i;\theta)-f(x^r_i))^2 \\
        &+ \frac{\beta|\partial\Omega|}{N_b} \sum_{j=1}^{N_b} (Tu_{nn}(x^b_j)-g(x^b_j))^2 \\
        &+ \frac{|\partial\Omega|}{N_b} \sum_{j=1}^{N_b} \lambda(x^b_j)(Tu_{nn}(x^b_j)-g(x^b_j)),
\end{align*}and the corresponding update rules are given as:
\begin{equation}
    \begin{aligned}\nonumber
        &\theta \leftarrow \theta - \eta_{\theta} \nabla_{\theta} \mathcal{L}_{\lambda}(\theta), \\
        &\lambda_j \leftarrow \lambda_j + \eta_{\lambda} \nabla_{\lambda_j} \mathcal{L}_{\lambda}(\theta),
    \end{aligned}
\end{equation} where $\lambda_j = \lambda(x^b_j)$, for $j=1,2,...,N_b$.
\subsection{General statements on convergence}\label{sec2.3}
The notion of $\Gamma'$-convergence of functionals, which is central to the established convergence theory, is introduced first. We modify Definition 4.1 in \cite{dal2012introduction} and define the $\Gamma'$-convergence as follows.

\begin{definition}
    Let $X$ be a topological space and $\{F_n\}_{n\in\mathbb{N}} : X \to (-\infty, \infty]$ be a sequence of functionals. Then, $\{F_n\}_{n\in\mathbb{N}}$ is said to $\Gamma'$-converge to $F : X \to (-\infty, \infty]$, if the following conditions are satisfied.
        \begin{enumerate}
            \item (Liminf inequality) For all $x \in F_A$, and for a sequence $\{x_n\}_{n\in\mathbb{N}}$ such that $x_n \rightharpoonup x$, we have $\displaystyle F(x) \leq \liminf_{n \to \infty} F_n(x_n)$.
            \item (Recovery sequence) For all $x \in X$, there exists a sequence $\{x_n\}_{n\in\mathbb{N}}$ such that $x_n \rightharpoonup x$, and $\displaystyle F(x) = \lim_{n \to \infty} F_n(x_n)$.
            \item (Admissible limit point) If $x$ is a limit point of sequence $\left\{ x_n \right\}_{n\in\mathbb{N}}$ with 
            \begin{align*}
                F_{n}(x_{n}) \le \inf_{x\in X}F_n(x_n)+\delta_{n}, && \text{for all} && n \in \mathbb{N}
            \end{align*} where $\delta_{n}\rightarrow 0$ (i.e. $\{x_{n}\}$ is a sequence of quasi-minimzers of $\{F_{n}\}$), then $x\in F_A$.
        \end{enumerate}
        Here, $F_A=\left\{x\in X| F(x)<\infty \right\}$ denotes the admissible set of $F$.
\end{definition}

    The following definition describes the equicoercivity of the sequence of functionals.
\begin{definition} [Definition 7.6 in \cite{dal2012introduction}]
    Let $X$ be a topological space and $\{F_n\}_{n\in\mathbb{N}} : X \to (-\infty, \infty]$ be a sequence of functionals defined on $X$. Then, we say that $\{F_n\}_{n\in\mathbb{N}}$ is equicoercive if, for every $r\in \mathbb{R}$, there exists a compact set $K_r \subset X$ such that 
    \begin{equation}
        \bigcup_{n\in\mathbb{N}} \{x\in X : F_n(x) \leq r\} \subset K_r. \nonumber
    \end{equation}    
\end{definition} 
Thus, the compactness argument can be applied to the sequence of minimizers of an equicoercive sequence of functionals. Indeed, the boundedness of minimizers ensures the existence of a convergent subsequence. The following theorem bridges the notion of $\Gamma'$-convergence with the convergence of dominant quasi-minimizers of the functionals.
\begin{restatable}{theorem}{Conv}\label{Conv}
    Let X be a reflexive Banach space and $\{F_n\}_{n\in\mathbb{N}}$ be a sequence of equicoercive functionals on X that $\Gamma'$-converges to $F$ with a unique minimizer $x$. Then, every sequence $\{x_n\}_{n\in\mathbb{N}}$ of quasi-minimizers of $\{F_n\}_{n\in\mathbb{N}}$ converges weakly to $x$. 
\end{restatable}

We remark that the Sobolev space, which is often referred to as the solution space of partial differential equations, is indeed a reflexive Banach space. Finally, we introduce a useful lemma to prove the above theorem.

\begin{restatable}{lemma}{quasi}\label{quasi}
Let $X$ be a topological space and $\{F_n\}_{n\in\mathbb{N}}$ be a sequence of functionals that $\Gamma'$-coverges to $F$. Let $\{x_n\}_{n\in\mathbb{N}}$ be a sequence of quasi-minimizers. If $x$ is a limit point of $\{x_n\}_{n\in\mathbb{N}}$, then $x$ is a minimizer of $F$ .
\end{restatable}

\section{Convergence Analysis}\label{sec3}

In this section, we provide theoretical justifications for the proposed method by showing its convergence to an actual solution. Proofs of the theorems in this section are provided in Section \ref{appendix_proof}. 

Consider a generic boundary value problem:
\begin{equation}\label{gov}
    \begin{aligned}
        &Nu = f, &&\text{ for } x\in\Omega, \\
        &Tu = g, &&\text{ for } x\in\partial\Omega,
    \end{aligned}
\end{equation} 
where $N$ and $T$ denote the differential and trace operator, respectively. Throughout this paper, we assume that $\Omega$ is a bounded, open, and connected subset of $\mathbb{R}^n$, where $\partial \Omega$ denotes the boundary of $\Omega$. If \eqref{gov} admits a unique strong solution (i.e., the differentiability of the solution is guaranteed as required in $N$ or $B$), then solving \eqref{gov} is equivalent to finding a minimizer of the following functional:
\begin{equation}\label{loss_infty}
    L(u)=\begin{cases}||Nu-f||^2_{L^2(\Omega)} & (Tu=g)\\ \infty \quad &(\text{otherwise})\end{cases}.
\end{equation}
In this study, we consider the following sequence of loss functionals $\left\{L_n \right\}_{n=1}^{\infty}$ that incorporates the proposed loss functions in \eqref{loss_function}.
\begin{eqnarray}\label{L_n}
    L_n(u) =  
    \footnotesize\text{$\begin{cases}
        ||Nu-f||_{L^2(\Omega)}^2 + \beta ||Tu-g||_{L^2{(\partial\Omega})}^2
        + \langle \lambda_n ,  Tu-g\rangle_{L^2{(\partial \Omega)}},  \\
        \hfill{\text{for } (u\in A_n)}  \\ 
        \infty, \hfill{\text{for } (u\notin A_n)}
    \end{cases}$}
\end{eqnarray}
\normalsize
where $A_n$ denotes a set of neural networks with width $n$ and depth $O(dim(\Omega))$. We will show that the sequence of minimizers of the loss functionals $L_n$ converges to an exact solution.

Based on Theorem \ref{Conv}, we plan to show that $\{L_n\}_{n=1}^{\infty}$ is an equicoercive sequence of functionals that $\Gamma'$-converges to $L$. We begin by showing the existence of minimizers of $L_n$ by referring to the universal approximation theorem for neural networks. Note that $C^m(\Omega)$ is a set of functions in which all partial derivatives of the order smaller than or equal to $m$ are continuous.

\begin{theorem}[Theorem 2.1 in \cite{li1996simultaneous}]\label{universal}
     Let $K$ be a compact subset of $\mathbb{R}^d$. Suppose there exists an open set $\Omega$ containing $K$ such that $f$ lies in $C^m(\Omega)$ for some $m\in Z_{+}^{d}$. If the activation function $\sigma$ is in $C^m(\mathbb{R})$, then for any $\varepsilon>0$, there exists a neural network $u_{nn}(x)=\sum_{i=1}^{h}c_i \sigma(w_i x+b_i)$ such that 
    \begin{equation}
        ||D^\alpha (u_{nn}) - D^\alpha (f)||_{L^{\infty}(K)}<\varepsilon, \forall \alpha \in Z_{+}^d  \text{ with } |\alpha|\le m. \nonumber
    \end{equation}
\end{theorem}
As a direct corollary of the theorem, the following proposition states that there always exists a neural network that makes the loss functional $L_n$ sufficiently small. 

\begin{restatable}{proposition}{losszero}\label{loss_zero}
    For a given $\epsilon >0,$ there exists a neural network $u_{nn}$ defined as in \ref{universal}, such that $L_n(u_{nn}) < \epsilon$, where $L_n$ is defined as in \eqref{L_n}, for the Helmholtz equation given in \eqref{Helmholtz_eq}. 
\end{restatable}

\begin{remark}
    The same argument can easily be made for the viscous Burgers equation \eqref{Burgers_eq} and the Klein--Gordon equation \eqref{Klein_eq}
\end{remark}

We next show that the sequence of proposed loss functions $\{L_n\}_{n\in\mathbb{N}}$ is an equicoercive sequence of functionals for three benchmark equations.

\begin{restatable}{theorem}{equi}\label{equi}
    The sequence $\{L_n\}_{n\in\mathbb{N}}$ is equicoercive in the weak topology for the Helmholtz \eqref{Helmholtz_eq}, viscous Burgers \eqref{Burgers_eq}, and Klein--Gordon equations \eqref{Klein_eq}.
\end{restatable}

To obtain the desired convergence result, we need the uniform boundedness of the sequence $\{\|\lambda_n\|_{L^2(\partial\Omega)}\}$. Next lemma states that for a sufficiently large penalty parameter $\beta$, we have the boundedness property of $\{\|\lambda_n\|_{L^2(\partial\Omega)}\}$.

\begin{restatable}{lemma}{Bound} \label{Bound}
    Consider a sequence of quasi-minimizers $\left\{u_n \right\}_{n\in\mathbb{N}}$ with respect to $\left\{ L_n \right\}_{n\in\mathbb{N}}$ for a sufficiently large $\beta$. Then, the corresponding $\|\lambda_n\|_{L^2(\partial\Omega)}$ is bounded for all $n\in\mathbb{N}$.
\end{restatable}

\begin{remark} \label{beta_reamrk}
    In convex optimization literature, a large value of $\beta$ stabilizes the augmented Lagrangian method. Since the uniform boundedness of $\{\|\lambda_n\|_{L^2(\partial\Omega)}\}$ is essential to guarantee the $\Gamma'$-convergence, a large $\beta$ stabilizes the proposed AL-PINNs as in convex optimization. We will numerically confirm the uniform boundedness of $\{\|\lambda_n\|_{L^2(\partial\Omega)}\}$ in Section \ref{sec4.1}.
\end{remark}

We are now ready to prove the convergence of dominant quasi-minimizers of the proposed loss functionals $\{L_n\}_{n\in\mathbb{N}}$. The following theorem states the convergence by proving the $\Gamma'$-convergence of the proposed loss functionals $\{L_n\}_{n\in\mathbb{N}}$ to $L$. 

\begin{restatable}{theorem}{convPDE}\label{conv_PDE}
    Consider a sequence of quasi-minimizers $\left\{u_{n} \right\}_{n\in \mathbb{N}}$ with respect to $\left\{L_{n}\right\}_{n \in \mathbb{N}}$. Then, $u_n \rightarrow u$ in the weak topology, where $u$ is the minimizer of functional $L$ (i.e., $u$ is the solution of equation \eqref{gov}). This holds for the Helmholtz \eqref{Helmholtz_eq}, viscous Burgers \eqref{Burgers_eq}, and Klein--Gordon equations \eqref{Klein_eq}.
\end{restatable}

\begin{proof}
    We begin the proof by showing that $u \in {L_A}$. Suppose that $u \notin {L_A}$. That is, $u$ is not a function satisfying $T(u)=g$ in $L^2(\partial \Omega)$. The trace operator $T$, which is a bounded and linear operator, admits an adjoint operator $T^{*} : H^1(\Omega) \to L^2(\partial \Omega)$. This implies that $T(u_n)\rightarrow T(u)\ne g$ weakly whenever $u_n \rightarrow u$ in the weak sense, from the below observation.
    \begin{align*}
        \langle T(u_n)-T(u), v\rangle_{L^2(\partial \Omega)} &= \langle T(u_n - u), v\rangle_{L^2(\partial \Omega)}\\&=\langle u_n - u, T^*(v)\rangle_{H_{0}^{1}(\Omega)}.
    \end{align*}
    Consequently, we obtain the following contradiction by the Cauchy--Schwarz inequality with Lemma \ref{Bound}.
    \begin{align*}
        \infty &= \liminf_{m \rightarrow \infty} \sum_{n=1}^{m}\langle T(u_n)-g, T(u)-g\rangle_{L^2(\partial \Omega)}\\ & = \liminf_{m \rightarrow \infty} \langle \sum_{n=1}^{m} T(u_n)-mg, T(u)-g\rangle_{L^2(\partial \Omega)} \\
        & \le \limsup_{m\rightarrow \infty} \|\sum_{n=1}^{m} T(u_n)-mg\|_{L^{2}(\partial \Omega)}\|T(u)-g\|_{L^2(\partial \Omega)}<\infty.
    \end{align*}
    Now we assume that $u \in L_A$, for a sequence $\{u_n\}_{n\in\mathbb{N}}$ with $u_n\rightarrow u$. The convexity of the mapping $x\mapsto |x|^2$ yields the inequality
    \begin {align*}
        \| T(u_n) -g\|^2_{L^2(\Omega)} \ge &\| T(u)-g\|^2_{L^2{(\Omega})}+ \\&2 \langle  T(u)-g ,  T(u) - T(u_n) \rangle _{L^2 (\Omega)}, 
    \end{align*} 
    where $\langle  T(u)-g, T(u)-T(u_n) \rangle _{L^2(\Omega)}\rightarrow 0$ by the weak convergence. By the Cauchy--Schwarz inequality, 
    \begin{equation*}
        \langle \lambda_n, T(u_n)-g \rangle_{L^2 (\partial\Omega)} \ge -||\lambda_n||_{L^2(\partial \Omega)} \cdot ||T(u_n)-g||_{L^2{(\partial\Omega})},
    \end{equation*} and therefore, we obtain the following inequality by the boundedness of $\|\lambda_n\|_{L^2(\partial \Omega)}$ with the lower semi-continuity of $L^2(\partial \Omega)$ norm.
    \begin{align*}
        L(u)= &\|Nu - f\|_{L^2(\Omega)} \\ \le&\liminf_{n\rightarrow \infty} (\|Nu_n -f\|_{L^2(\Omega)}^2 +\beta\|T(u_n)-g\|^2_{L^2(\partial\Omega)}+\\&\langle \lambda_n ,  Tu_n-g\rangle_{L^2{(\partial \Omega)}})\\=&\liminf_{n\rightarrow \infty} L_n (u_n)
    \end{align*}
    Next, we prove that a recovery sequence exists. For a given sequence $\left\{\lambda_n\right\}_{n\in \mathbb{N}}$ and $u\in L_A$, a sequence $\left\{u_n\right\}_{n\in N}$ that satisfies the following inequality holds in $A_n$.
    \begin{align}\nonumber
        \limsup_{n\rightarrow \infty} L_n (u_n )\le L(u).
    \end{align}
    If $u\notin L_A$, consider any sequence $\left\{u_n \right\}_{n\in \mathbb{N}}$ in $A_n$. Since $F(u)=\infty$, it is obvious that the above inequality holds. Conversely, suppose that $u\in L_A$. By Theorem \ref{universal}, there exists a sequence $\left\{u_n \right\}_{n\in \mathbb{N}}$ in $A_n$ such that $\|Nu-Nu_n\|_{V (\Omega)}<{1}/{n}$. Then $\left\{u_n \right\}_{n \in \mathbb{N}}$ is a desired sequence since $T(u_n)\rightarrow T(u)$ in $L^2(\partial \Omega)$ by the continuity of the trace operator so that
    \begin{align*}
        \limsup_{n\rightarrow\infty} L_n(u_n)=&\limsup_{n\rightarrow \infty} (\|Nu_n -f\|_{L^2(\Omega)}^2 +\beta\|T(u_n)-g\|^2_{L^2(\partial\Omega)}+\\& \langle \lambda_n ,  Tu_n-g\rangle_{L^2{(\partial \Omega)}})\\\leq&\limsup_{n\rightarrow\infty}(||Nu -f||_{L^2(\Omega)}^2+\frac{1}{n})=L(u).
    \end{align*}
    In conclusion, $u_{n}\rightarrow u$ weakly in $H^{1}(\Omega)$ by Theorem \ref{Conv}. By Rellich-embedding theorem from $W^{1,2}$ to $L^2$, the embedding of $H^1(\Omega)$ into $L^2(\Omega)$ is a compact operator so that it maps a weakly convergent sequence to a strongly convergent sequence. Therefore, $u_{n}\rightarrow u$ in $L^2(\Omega)$ strongly. 
\end{proof}

\begin{remark}
    Although we prove the convergence for three benchmark equations, we expect that some analogous results can be derived for a variety of PDEs. It would be sufficient to verify that the proposed sequence of loss functional is equi-coercive and $\Gamma'-$ convergent to the objective functional. Further discussion on linear elliptic equations is briefly introduced in \ref{appendix_proof}.
\end{remark}


\section{Numerical Experiments}\label{sec4} 
We detail our experimental settings and results that demonstrate the superior performance of the proposed AL-PINNs from two different perspectives. First, we compare the approximation error of the AL-PINNs with that of other constrained optimization strategies, such as the penalty and Lagrange multiplier methods. In this comparative analysis, we observe that both the penalty and multiplier terms are necessary for obtaining an accurate approximation. Next, we compare the approximation error of the proposed method with that of several adaptive loss-balancing algorithms. In this experiment, we provide evidence that AL-PINNs approximate the solution accurately, whereas the existing adaptive loss-balancing algorithms fail to reduce the boundary error sufficiently.

Throughout this section, we denote the number of points for the residual loss by $N_r$, the number of points for the boundary conditions by $N_B$, and the number of points for the initial conditions by $N_I$. We employ the hyperbolic tangent function as an activation function, and the ADAM optimizer proposed in \cite{kingma2014adam} for training. 

\subsection{AL-PINNs as a deep constrained optimization method for PINNs}\label{sec4.1}
\textbf{Experimental Setup.} In this subsection, we present several experimental results that demonstrate the superior performance of AL-PINNs compared to both the penalty and standard Lagrange multiplier methods (see Section \ref{sec2} for algorithms). We measure the relative $L^2$ error of the neural network solution $u_{nn}$, given by $\|u-u_{nn}\|_{L^2}/\|u\|_{L^2}$, for each algorithm, using the Helmholtz equation as a benchmark PDE. The equation is as follows: \begin{equation}
\begin{aligned}\nonumber
    &\Delta u + u = f(x,y), &&\text{ for } (x, y) \in \Omega, \\
    &u(x,y) = 0, &&\text{ for } (x,y) \in \partial\Omega, 
\end{aligned}
\end{equation}
where $\Omega = [-1,1]\times[-1,1]$. If we take 
\begin{align}
    f(x,y) &= -\pi^2 \sin(\pi x)\sin(4\pi y) \nonumber \\
    &- (4\pi)^2 \sin(\pi x)\sin(4\pi y) + \sin(\pi x)\sin(4\pi y), \nonumber
\end{align}
then it can readily be shown that $u(x,y)=\sin(\pi x)\sin(4\pi y)$ is an analytic solution. The loss functions are given as:
\begin{equation}
\begin{aligned}
    \mathcal{L}_{\beta_n}^{(P)} &\approx \frac{1}{N_r}\sum_{i=1}^{N_r}(\Delta u_{nn}(x_i) + u(x_i) - f(x_i))^2  \nonumber\\
        &+ \frac{\beta_n}{N_B}\sum_{j=1}^{N_b} u_{nn}^2(x_j),\nonumber\\
        \mathcal{L}_{\lambda_n}^{(L)} &\approx \frac{1}{N_r}\sum_{i=1}^{N_r}(\Delta u_{nn}(x_i) + u(x_i) - f(x_i))^2  \nonumber\\
        &+ \frac{1}{N_B} \sum_{j=1}^{N_B}\lambda_n(x_j)u_{nn}(x_j),\nonumber\\
        \mathcal{L}_{\lambda_n, \beta}^{(A)} &\approx \frac{1}{N_r}\sum_{i=1}^{N_r}(\Delta u_{nn}(x_i) + u(x_i) - f(x_i))^2  \\
        &+ \frac{\beta}{N_B}\sum_{j=1}^{N_b} u_{nn}^2(x_j) + \frac{1}{N_B} \sum_{j=1}^{N_B}\lambda_n(x_j)u_{nn}(x_j), \nonumber
\end{aligned}
\end{equation}
where $\mathcal{L}_{\beta_n}^{(P)}$ denotes a loss function for the penalty method, $ \mathcal{L}_{\beta_n}^{(L)}$ denotes that of the Lagrange multiplier method, $ \mathcal{L}_{\beta_n}^{(A)}$ denotes that of AL-PINNs, $\beta_n \to \infty$, and $\beta$ is a predefined constant. In this subsection, the layers of the neural networks consist of neurons 2-256-256-1.

\textbf{Results.} Figure \ref{loss_error} shows the trajectories of the total loss given in \eqref{pinn_loss} and the relative $L^2$ error during training with the penalty method with a linearly increasing sequence of $\{\beta_n\}_{n=1}^{50000}$, the Lagrange multiplier method, and our AL-PINNs with $\beta=1$. We observe that a neural network hardly learns the solution of the Helmholtz equation when training with the Lagrange multiplier method. This finding coincides with the theoretical observation made in Remark \ref{beta_reamrk} that there exists a lower bound for $\beta$ to achieve the boundedness of the sequence $\{\lambda_n\}$. We also observe that both the total loss and the relative error rapidly converge to relatively high values as $\beta_n$ increases. This phenomenon is somewhat explained by the fact that a large value of $\beta$ causes the dominance of the boundary condition over the residual loss, which makes the training unstable. Finally, it is evident that our AL-PINNs yield a significantly smaller relative error (approximately 40 times smaller) compared to the penalty method. 

\begin{figure}[H]
\begin{center}
\centerline{\includegraphics[width=1\columnwidth,height=0.4\columnwidth,draft=False]{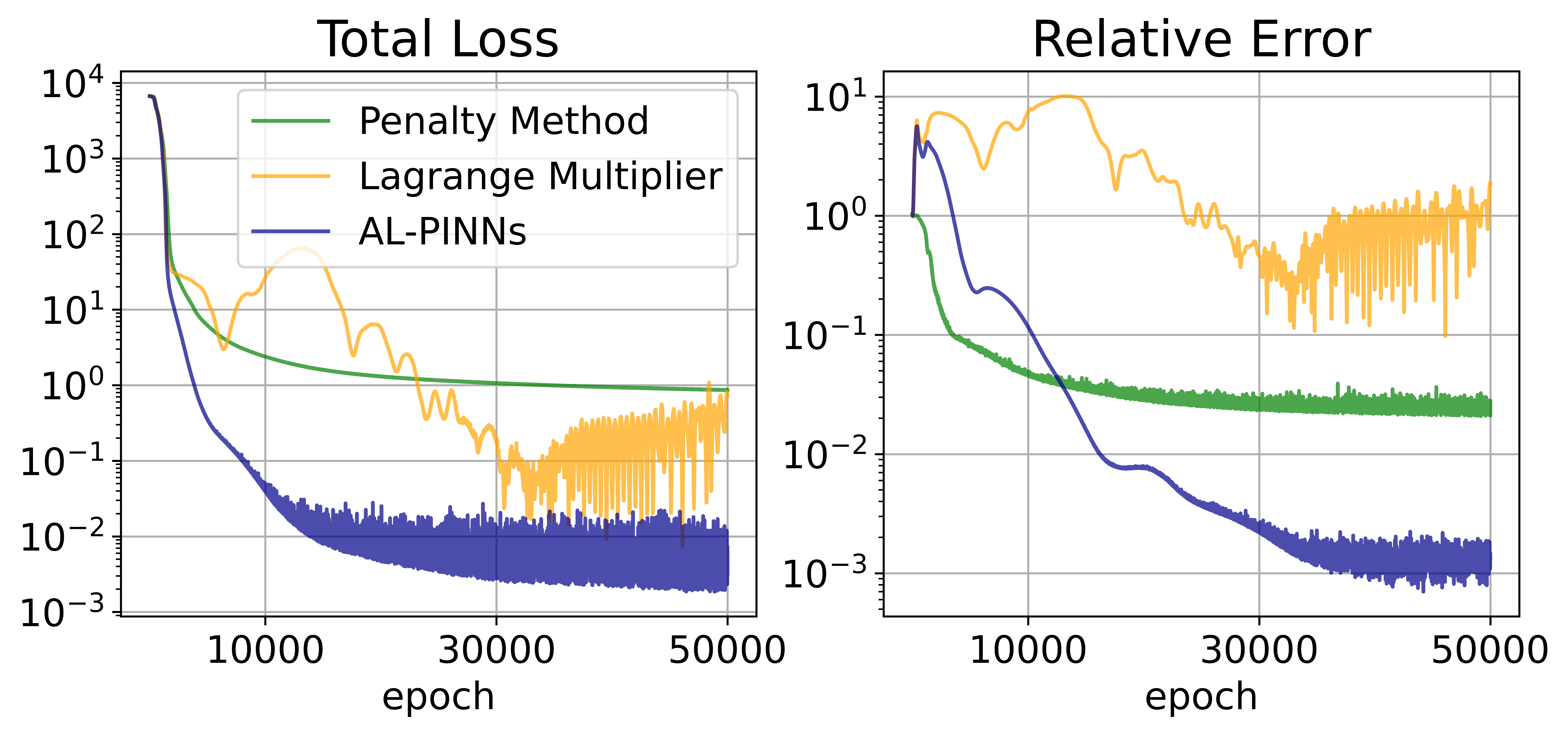}}
\caption{Left: Trajectories of the total loss given in \eqref{pinn_loss} for the penalty method (green), Lagrange multiplier method (orange), and the proposed AL-PINNs (blue). Right: Trajectories of the relative $L^2$ error for the same algorithms. Here, we set $\beta_n = 10n$ for the penalty method, $\eta_\lambda = 10^{-4}$ for the Lagrange multiplier method, and $\eta_\lambda = 10^{-4}, \beta=1$ for the proposed AL-PINNs.}
\label{loss_error}
\end{center}
\vskip -0.2in
\end{figure}

Figure \ref{err_heatmap} shows the analytic solution and pointwise absolute errors for the training algorithms. Due to the training instability of the penalty method, we set $\beta_n \equiv 1$ as in the original work by \cite{raissi2019physics}. We observe that most errors arise from the near boundary points in both the penalty and Lagrange multiplier methods. This indicates that the boundary condition is indeed a critical factor for a relatively high error level. As it exhibits negligible boundary errors, we can conclude that the proposed loss function given in \eqref{loss_function} is the best relaxation method of the boundary condition into the loss function.

\begin{figure}[H]
\begin{center}
\centerline{\includegraphics[width=\columnwidth,draft=False]{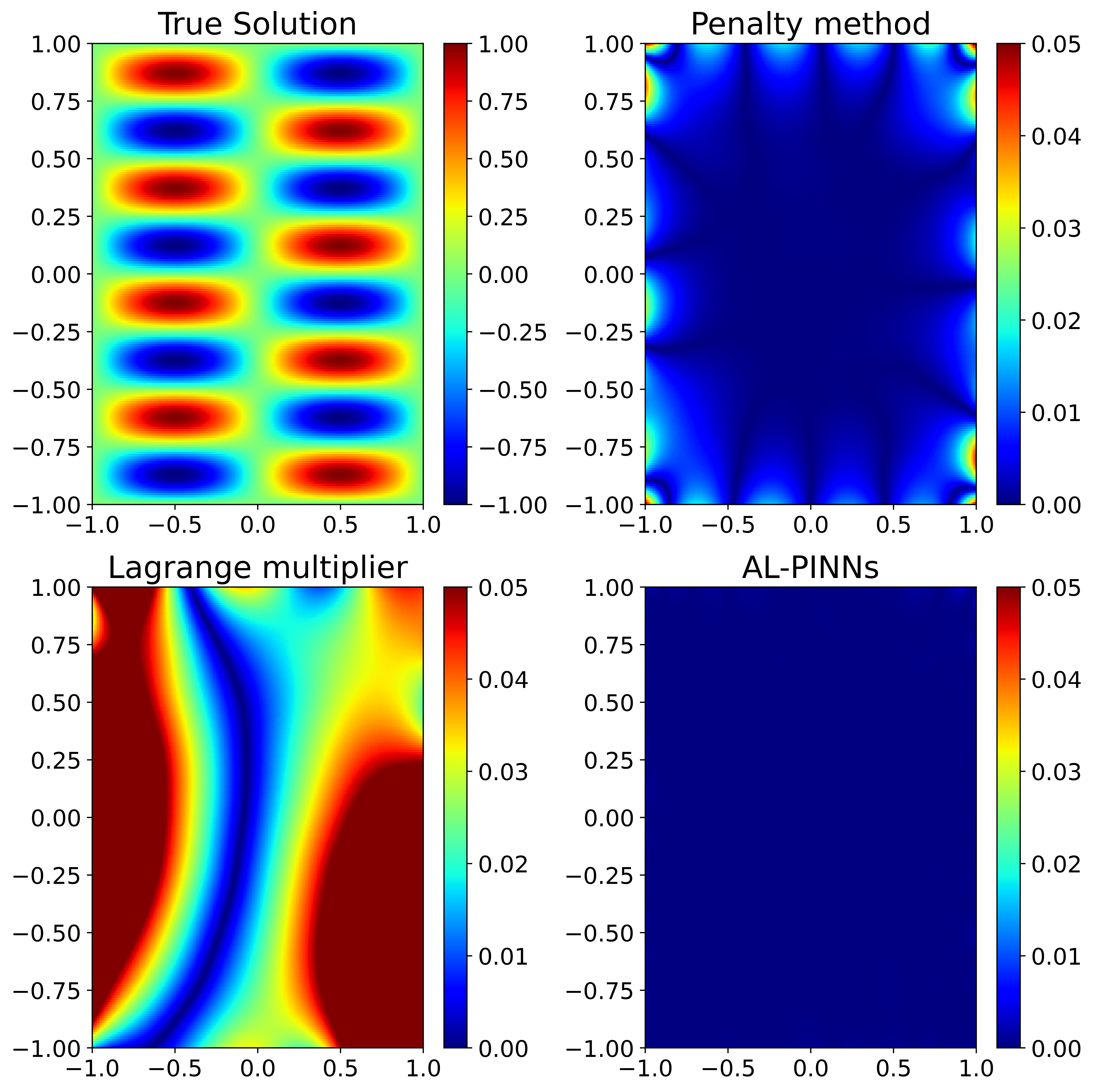}}
\caption{Top left: True solution of the Helmholtz equation. Top right: Pointwise absolute error for the penalty method with $\beta_n \equiv 1$. Bottom left: Pointwise absolute error for the Lagrange multiplier method with $\eta_{\lambda}=10^{-4}$. Bottom right: Pointwise absolute error for the proposed AL-PINNs with $\eta_{\lambda} = 10^{-4}, \beta=1$. Errors are computed using the best model over 50,000 training epochs.}
\label{err_heatmap}
\end{center}
\vskip -0.2in
\end{figure}

$\beta$ plays a significant role in training a neural network with both the penalty method and our AL-PINNs. In Figure \ref{beta_comparison}, we provide the relative $L^2$ errors for the penalty method with a constant penalty parameter and our AL-PINNs for different values of $\beta$. The left panel of Figure \ref{beta_comparison} shows the relative errors for different values of $\beta$ after 50,000 training epochs. One can see that the relative errors highly depend on the values of $\beta$ in both cases, yet the proposed AL-PINNs uniformly outperform (up to 70 times) the penalty method. The right panel of Figure \ref{beta_comparison} supports the uniform boundedness argument of $\{\|\lambda_n\|_{L^2(\partial\Omega)}\}$ stated in Lemma \ref{Bound} throughout the training. As we can see in Figure \ref{beta_comparison}, the $L^2$ norm of $\lambda_n$ converges to a certain value for all $\beta$ in $\{1,10,100,1000\}$ even with a relatively high value of $\eta_{\lambda}=1$. Thus, we can say that the sequence $\{\lambda_n\}$ has an upper bound, and an important condition to guarantee the convergence is satisfied.

\begin{figure}[H]
\begin{center}
\centerline{\includegraphics[width=\columnwidth,height=0.4\columnwidth,draft=False]{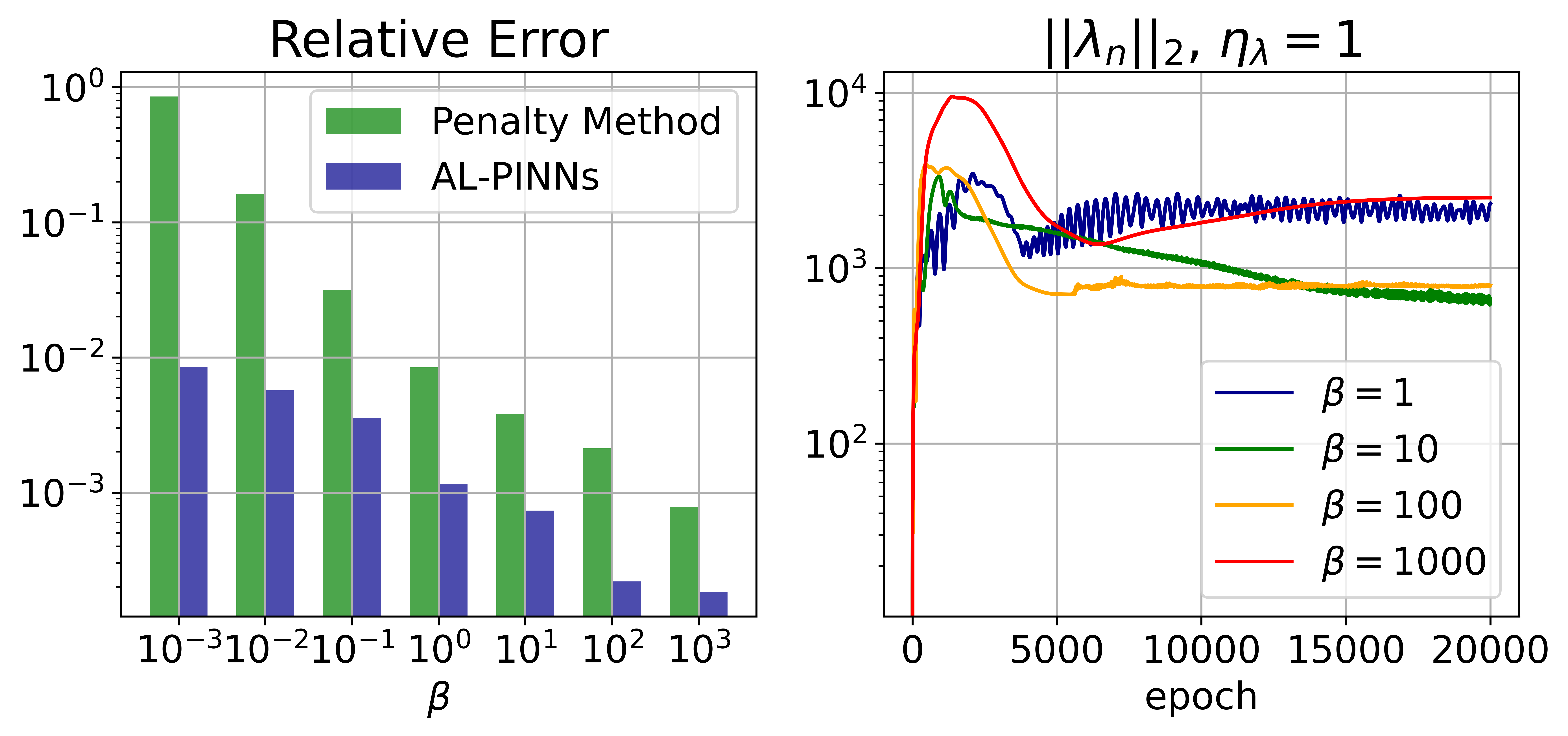}}
\caption{Left: Relative $L^2$ errors for the penalty method with constant $\beta_n \equiv \beta$ and the proposed AL-PINNs with the same $\beta$. Right: The $L^2$ norm of $\lambda_n$ in training epoch for different $\beta$'s.}
\label{beta_comparison}
\end{center}
\vskip -0.2in
\end{figure}

\textbf{Lower Bound for $\beta$.} Lemma \ref{Bound} states that $\|\lambda_n\|_{L^2}$ is uniformly bounded if $\beta$ is large enough. Here, we empirically investigate the question "which $\beta$ is large enough?". Figure \ref{beta_lowerbound} shows the value of $\|\lambda_n\|_{L^2}$ during training for three benchmark equations in section \ref{appendix_pde}. Interestingly, we observed that the values of $\|\lambda_n\|_{L^2}$ seem to be uniformly bounded, in all three cases, although $\beta$ is small enough. This implies that our assumption in Lemma \ref{Bound} saying "$\beta$ is large enough" is only a sufficient condition and there is room for improvement. 

\begin{figure}[H]
\begin{center}
\centerline{\includegraphics[width=\columnwidth,height=0.3\columnwidth,draft=False]{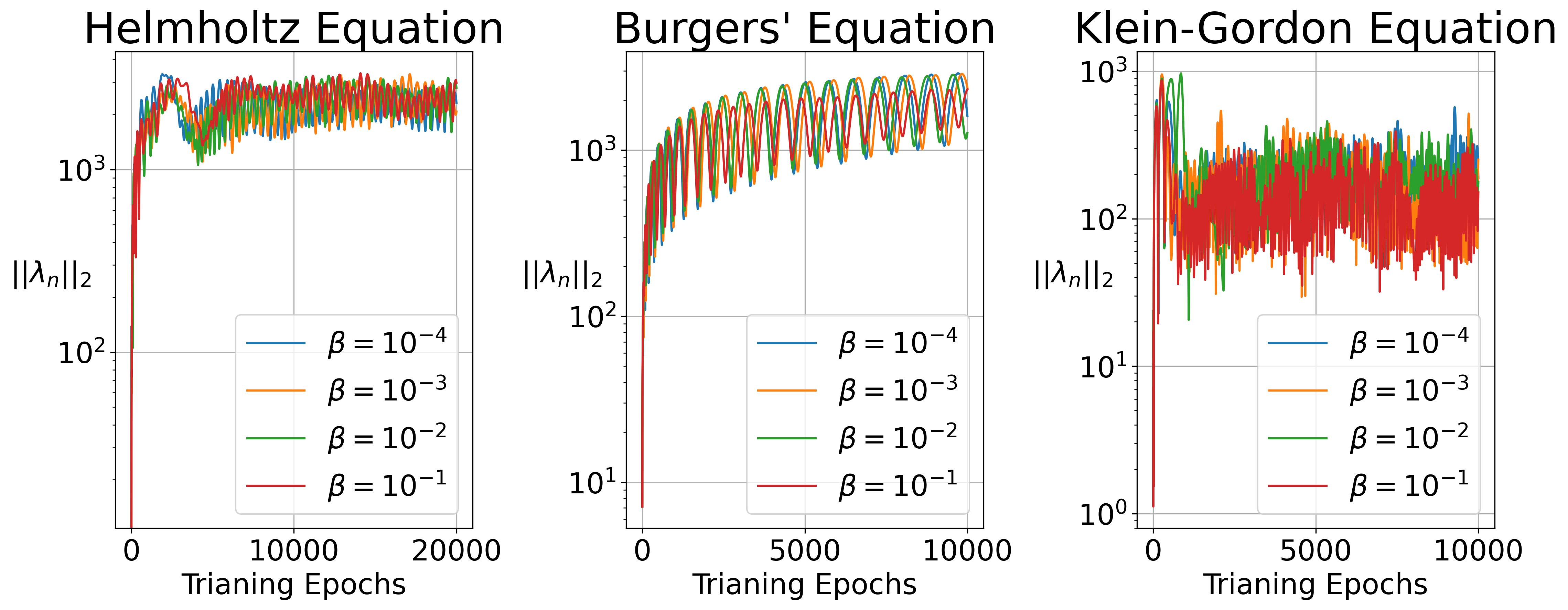}}
\caption{$L^2$ norms of $\lambda_n$ in training epochs for three benchmark equations.}
\label{beta_lowerbound}
\end{center}
\vskip -0.2in
\end{figure}

\subsection{AL-PINNs as an adaptive loss-balancing Algorithm for PINNs}\label{sec4.2}

\begin{table*}[ht]
\centering
\caption{Average relative $L^2$ errors and standard deviations over 10 different trials for the Helmholtz equation.}
\label{table_helmholtz}
\vskip 0.1in
\resizebox{2\columnwidth}{!}{
\begin{tabular}{|c|ccccc|}
\hline
\multirow{2}{*}{} & \multicolumn{5}{c|}{Helmholtz equation}                                                                                                                                                           \\ \cline{2-6} 
                  & \multicolumn{1}{c|}{PINNs}                    & \multicolumn{1}{c|}{SA}                      & \multicolumn{1}{c|}{LRA}                     & \multicolumn{1}{c|}{NTK}   & AL-PINNs                      \\ \hline
M1                & \multicolumn{1}{c|}{5.12e-01 $\pm$ 2.08e-01} & \multicolumn{1}{c|}{2.32e-01 $\pm$ 5.13e-02} & \multicolumn{1}{c|}{2.85e-01 $\pm$ 8.50e-02} & \multicolumn{1}{c|}{3.80e-01 $\pm$ 1.09e-01} & \textbf{6.44e-03 $\pm$ 3.65e-03} \\ \hline
M2                & \multicolumn{1}{c|}{2.11e-02 $\pm$ 4.82e-03} & \multicolumn{1}{c|}{2.40e-02 $\pm$ 2.26e-03} & \multicolumn{1}{c|}{1.06e-02 $\pm$ 1.91e-03} & \multicolumn{1}{c|}{2.18e-02 $\pm$ 6.05e-03} & \textbf{6.00e-04 $\pm$ 1.13e-04} \\ \hline
M3                & \multicolumn{1}{c|}{1.14e-01 $\pm$ 3.28e-02} & \multicolumn{1}{c|}{9.65e-02 $\pm$ 2.27e-02} & \multicolumn{1}{c|}{3.61e-02 $\pm$ 9.64e-03} & \multicolumn{1}{c|}{1.23e-01 $\pm$ 1.88e-02} & \textbf{4.90e-04 $\pm$ 6.47e-05} \\ \hline
M4                & \multicolumn{1}{c|}{1.93e-02 $\pm$ 4.78e-03} & \multicolumn{1}{c|}{2.43e-02 $\pm$ 3.18e-03} & \multicolumn{1}{c|}{8.89e-03 $\pm$ 1.05e-03} & \multicolumn{1}{c|}{2.25e-02 $\pm$ 4.16e-03} & \textbf{7.46e-04 $\pm$ 1.10e-04} \\ \hline
\end{tabular}}

\centering
\caption{Average relative $L^2$ errors and standard deviations over 10 different trials for the Burgers equation.}
\label{table_Burgers}
\vskip 0.1in
\resizebox{2\columnwidth}{!}{
\begin{tabular}{|c|ccccc|}
\hline
\multirow{2}{*}{} & \multicolumn{5}{c|}{Burgers' equation}                                                                                                                                                                               \\ \cline{2-6} 
                  & \multicolumn{1}{c|}{PINNs}                    & \multicolumn{1}{c|}{SA}                      & \multicolumn{1}{c|}{LRA}                     & \multicolumn{1}{c|}{NTK}                     & AL-PINNs                  \\ \hline
M1                & \multicolumn{1}{c|}{8.92e-02 $\pm$ 2.43e-02} & \multicolumn{1}{c|}{9.80e-02 $\pm$ 5.54e-02} & \multicolumn{1}{c|}{1.48e-01 $\pm$ 4.56e-02} & \multicolumn{1}{c|}{1.22e-01 $\pm$ 3.58e-02} & \textbf{5.45e-02 $\pm$ 9.41e-03} \\ \hline
M2                & \multicolumn{1}{c|}{7.29e-02 $\pm$ 7.14e-03} & \multicolumn{1}{c|}{8.40e-02 $\pm$ 8.48e-03} & \multicolumn{1}{c|}{6.64e-02 $\pm$ 5.86e-03} & \multicolumn{1}{c|}{6.61e-02 $\pm$ 8.13e-03} & \textbf{5.91e-02 $\pm$ 6.37e-03} \\ \hline
M3                & \multicolumn{1}{c|}{4.85e-02 $\pm$ 5.95e-03} & \multicolumn{1}{c|}{4.71e-02 $\pm$ 1.82e-02} & \multicolumn{1}{c|}{4.38e-02 $\pm$ 6.39e-03} & \multicolumn{1}{c|}{4.52e-02 $\pm$ 6.50e-03} & \textbf{4.10e-02 $\pm$ 7.32e-03} \\ \hline
M4                & \multicolumn{1}{c|}{1.21e-01 $\pm$ 3.26e-02} & \multicolumn{1}{c|}{1.06e-01 $\pm$ 2.16e-02} & \multicolumn{1}{c|}{6.34e-02 $\pm$ 6.39e-03} & \multicolumn{1}{c|}{7.40e-02 $\pm$ 1.21e-02} & \textbf{5.89e-02 $\pm$ 4.22e-03} \\ \hline
\end{tabular}}

\centering
\caption{Average relative $L^2$ errors and standard deviations over 10 different trials for the Klein--Gordon equation.}
\label{table_KG}
\vskip 0.1in
\resizebox{2\columnwidth}{!}{
\begin{tabular}{|c|ccccc|}
\hline
\multirow{2}{*}{} & \multicolumn{5}{c|}{Klein--Gordon equation}                                                                                                                                                       \\ \cline{2-6} 
                  & \multicolumn{1}{c|}{PINNs}                    & \multicolumn{1}{c|}{SA}                      & \multicolumn{1}{c|}{LRA}                     & \multicolumn{1}{c|}{NTK}   & AL-PINNs                      \\ \hline
M1                & \multicolumn{1}{c|}{3.86e-01 $\pm$ 1.21e-01} & \multicolumn{1}{c|}{2.45e-01 $\pm$ 1.23e-01} & \multicolumn{1}{c|}{2.39e-01 $\pm$ 4.62e-02} & \multicolumn{1}{c|}{8.11e-01 $\pm$ 2.77e-01} & \textbf{1.42e-02 $\pm$ 7.34e-03} \\ \hline
M2                & \multicolumn{1}{c|}{5.25e-02 $\pm$ 1.42e-02} & \multicolumn{1}{c|}{4.32e-02 $\pm$ 1.21e-02} & \multicolumn{1}{c|}{2.22e-02 $\pm$ 1.30e-02} & \multicolumn{1}{c|}{1.33e-02 $\pm$ 6.80e-03} & \textbf{5.73e-03 $\pm$ 1.45e-03} \\ \hline
M3                & \multicolumn{1}{c|}{1.10e-01 $\pm$ 4.93e-02} & \multicolumn{1}{c|}{1.40e-01 $\pm$ 4.77e-02} & \multicolumn{1}{c|}{7.19e-02 $\pm$ 2.56e-02} & \multicolumn{1}{c|}{3.01e-02 $\pm$ 1.17e-02} & \textbf{7.35e-03 $\pm$ 1.98e-03} \\ \hline
M4                & \multicolumn{1}{c|}{5.74e-02 $\pm$ 1.76e-02} & \multicolumn{1}{c|}{3.65e-02 $\pm$ 1.09e-02} & \multicolumn{1}{c|}{2.35e-02 $\pm$ 1.26e-02} & \multicolumn{1}{c|}{1.24e-02 $\pm$ 3.88e-03} & \textbf{5.28e-03 $\pm$ 1.37e-03} \\ \hline
\end{tabular}}
\end{table*}

\textbf{Baselines.} One can readily see that the proposed method naturally belongs to a class of adaptive loss-balancing algorithms. In this subsection, we demonstrate the superior performance of the proposed AL-PINNs compared with the vanilla Physics-Informed Neural Networks (PINNs), a Soft Attention mechanism (SA) proposed in \cite{mcclenny2020self}, a Learning Rate Annealing algorithm (LRA) presented in \cite{wang2021understanding}, and a loss-balancing algorithm via the eigenvalues of the Neural Tangent Kernel (NTK) proposed in \cite{wang2022and}. We use the Helmholtz, viscous Burgers, and the Klein--Gordon equations as benchmark PDEs, as they are widely used for this purpose in PINNs literature (For example, see \cite{son2021sobolev, mcclenny2020self, wang2021understanding, bischof2021multi}). Although \cite{wang2022and} did not investigate those equations, the idea is easily generalizable to those equations. We compare the relative $L^2$ error of the proposed AL-PINNs with those of the vanilla PINNs, SA, LRA, and NTK algorithms using the above PDEs. We provide detailed equations and loss functions in Appendix \ref{appendix_pde}.

\textbf{Experimental Setup.} We compare the algorithms for four different neural network architectures, namely M1, M2, M3, and M4 in Tables \ref{table_helmholtz}--\ref{table_KG}. M1 consists of 8 hidden layers with 64 neurons, M2 consists of 2 hidden layers with 256 neurons, M3 denotes M1 equipped with residual connections, and M4 denotes M2 equipped with residual connections. We uniformly sampled the test dataset from the domain of each PDE and computed the relative $L^2$ error of the neural network solution $u_{nn}$, given by $\|u-u_{nn}\|_{L^2}/\|u\|_{L^2}$, on the test dataset. For each training algorithm-architecture pair, we train 10 instances of neural networks with the Kaiming uniform initialization method presented in \cite{he2015delving}. We report the average relative $L^2$ errors and the standard deviations across 10 trials. The hyperparameter configurations for $\beta, \eta_{\theta}, \eta_{\lambda}$ are provided in appendix \ref{appendix_hyper}.


\textbf{Helmholtz equation.} We define the loss function $\mathcal{L}_{\lambda_n, \beta}^{(A)}$ as in \eqref{Helmholtz_ALM}, by using the proposed AL-PINNs. We train a neural network on a fixed uniform rectangular grid where $N_r=2500$ and $N_B=200$. We train the neural networks for 10000 epochs, and use the early stopping strategy as the stopping criteria. We summarized the average test errors and standard deviations for the Helmholtz equation in Table \ref{table_helmholtz}.

\begin{figure}[ht]
\begin{center}
\centerline{\includegraphics[width=\columnwidth,draft=False]{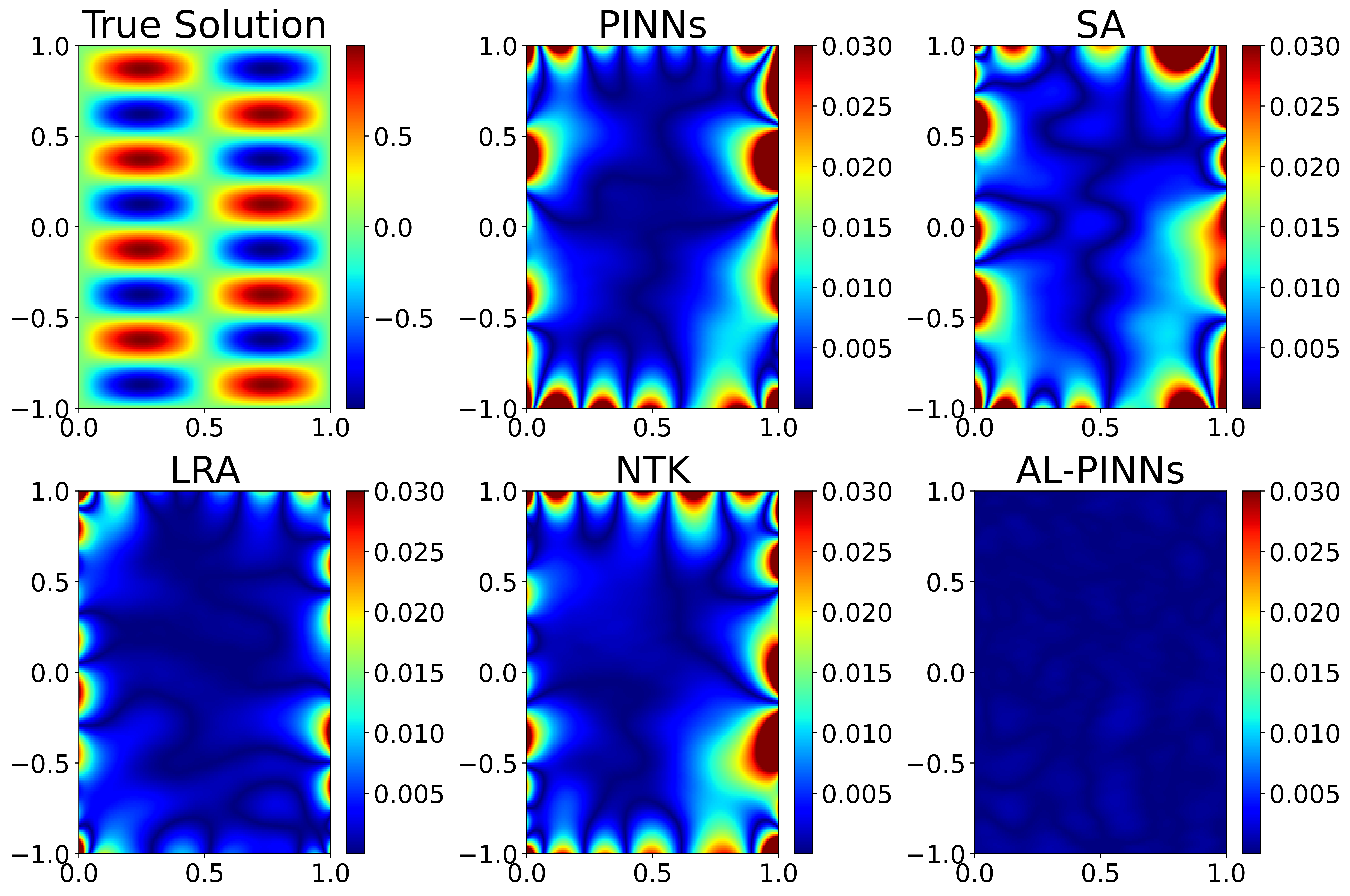}}
\caption{Pointwise absolute errors of the baseline algorithms for the Helmholtz equation. All adaptive loss-balancing algorithms suffer from boundary errors except our AL-PINNs.}
\label{helmholtz_loss_balancing_heatmap}
\end{center}
\vskip -0.2in
\end{figure}

Figure \ref{helmholtz_loss_balancing_heatmap} shows the analytic solution of the Helmholtz equation and pointwise absolute errors for the baseline algorithms with the model M2. All adaptive loss-balancing algorithms result in severe boundary errors except for our AL-PINNs. This implies that existing adaptive loss-balancing algorithms fail to find optimal $\lambda$'s to achieve an accurate approximation. On the other hand, the proposed AL-PINNs converges to a highly accurate approximate solution with a uniform error distribution. Table \ref{table_helmholtz} shows that the proposed AL-PINNs achieves a much smaller relative error than other adaptive loss-balancing algorithms throughout the models M1-M4.

\textbf{Viscous Burgers equation.} The viscous Burgers' equation \eqref{burgers_equation} admits an analytic solution presented in \cite{basdevant1986spectral}, which we use to compute the relative $L^2$ error. We define the loss function $\mathcal{L}_{\lambda_n, \beta}^{(A)}$ as in \eqref{Burgers_ALM} by using the proposed AL-PINNs.  We train the neural networks on a fixed uniform rectangular grid with $N_r = 2500, N_B = 100$, and $N_I = 50$, for 10,000 epochs, and employ the early stopping strategy as the stopping criteria.

Figure \ref{burgers_err_heatmap} shows the analytic solution and pointwise absolute errors for the vanilla PINNs, an adaptive loss-balancing algorithm using a soft attention mechanism (SA), and the proposed AL-PINNs with the model M1. \cite{mcclenny2020self} argued that PINNs suffers from an accuracy problem where the solution has a sharp spatio-temporal transition. For example, the solution of the viscous Burgers equation \eqref{burgers_equation} exhibits a sharp transition near $x=0$ (the horizontal line). The absolute errors in Figure \ref{burgers_err_heatmap} show that the proposed AL-PINNs results in a much smaller error near $x=0$, compared to other methods. This result demonstrates that the proposed AL-PINNs outperforms both the vanilla PINNs and the soft attention mechanism in a problem with sharp spatio-temporal transitions. We summarized the average relative $L^2$ errors and standard deviations in Table \ref{table_Burgers}.

\begin{figure}[]
\begin{center}
\centerline{\includegraphics[width=\columnwidth,draft=False]{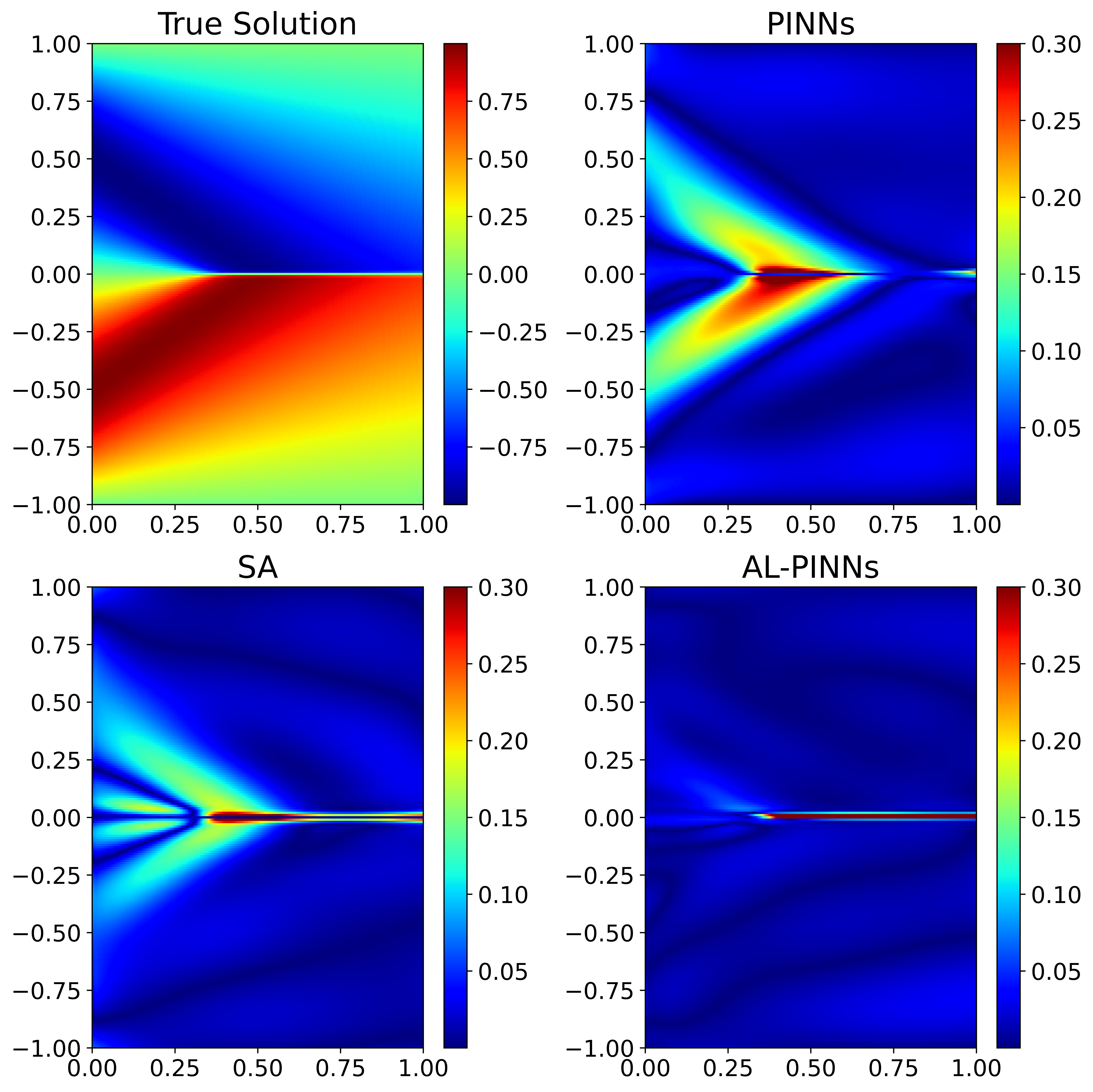}}
\caption{Top left: True solution of the viscous Burgers equation. Top right: Pointwise absolute error for the vanilla PINNs. Bottom left: Pointwise absolute error for the soft attention mechanism (SA). Bottom right: Pointwise absolute error for our AL-PINNs.}
\label{burgers_err_heatmap}
\end{center}
\vskip -0.2in
\end{figure}

\textbf{Klein--Gordon equation.} The equation and proposed loss function are given in \eqref{Klein_Gordon_equation} and \eqref{Klein-Gordon_ALM}, respectively. We train a neural network on a fixed uniform mesh with $N_r=2500, N_B=100$, and $N_I=50$. We train the neural networks for 10,000 epochs and employ the early stopping strategy as a stopping criterion. We summarize the results in Table \ref{table_KG}, which shows that the proposed AL-PINNs outperform other loss-balancing algorithms for all network architectures we considered.

\subsection{AL-PINNs with sinusoidal features}\label{sec4.3}
\textbf{Experimental Setup.} Recently, there have been a number of works that try to utilize sinusoidal feature mappings. For example, for the training of neural networks \cite{sitzmann2020implicit, tancik2020fourier}, and for the training of physics-informed neural networks \cite{wang2021eigenvector, wong2022learning} to name a few. In this subsection, we investigate the effect of the proposed AL-PINNs when applied to sinusoidal-feature physics-informed neural networks (sf-PINNs) proposed in \cite{wong2022learning}. We compare four models, a fully connected network with the original loss functions (PINNs), a fully connected network with the proposed loss functions (AL-PINNs), sinusoidal features with the original loss functions (sf-PINNs), and sinusoidal features with the proposed loss functions (AL-sf-PINNs). We fixed the network architecture with [$(t,x,y)$-64-50-50-50-[50-50-50-$\hat{u}$, 50-50-50-$\hat{v}$, 50-50-50-$\hat{p}$]] as in \cite{wong2022learning}. We use the hyperbolic tangent as activation functions and the layers are initialized via Xavier initialization. We computed the mean square error on a pre-defined test set which is uniformly sampled from the domain.

\textbf{2D transient Navier-Stokes equation.} We use the Navier-Stokes equation, describing the motion of viscous fluid, as a benchmark problem for this purpose. The equation and the exact closed form solution are given in \ref{Navier_Stokes_equation}. Table \ref{Navier_Stokes_table} shows the mean square errors and their standard deviations over 10 trials. As the exact solution consists of sines and cosines, sinusoidal feature mapping shows quite an improvement in the mean square error. Interestingly, we observe that AL-PINNs also shows quite an improvement compared to PINNs. We think that the proposed loss function enables the network to learn the sinusoidal features from the initial and boundary conditions. Finally, training a network with sinusoidal feature mappings with the proposed loss function (AL-sf-PINNs) results in the best mean square error, showing that the proposed loss function can be applied to a variety of networks.

\begin{table}[h]
\centering
\caption{Mean square errors and standard deviations over 10 trials for the Navier-Stokes equation.}
\label{Navier_Stokes_table}
\begin{tabular}{|c|c|}
\hline
            & MSE                     \\ \hline
PINNs       & 2.87e-04 $\pm$ 4.51e-05 \\ \hline
AL-PINNs    & 5.53e-05 $\pm$ 1.32e-05 \\ \hline
sf-PINNs    & 2.71e-05 $\pm$ 2.58e-05 \\ \hline
AL-sf-PINNs & \textbf{2.45e-06 $\pm$ 5.58e-06} \\ \hline
\end{tabular}
\end{table}

\section{Conclusions}\label{sec5}
In this paper, we proposed AL-PINNs, a convergence-guaranteed highly accurate adaptive loss-balancing algorithm for PINNs. We proved the convergence of a sequence generated by the proposed method in Section \ref{sec3}. We evaluated the proposed method in two different aspects. In Section \ref{sec4.1}, we observed that the penalty and multiplier terms should both be considered in the loss function to obtain an accurate approximation. In Section \ref{sec4.2}, we demonstrated the superior performance of AL-PINNs compared with other loss-balancing algorithms in various settings. We also investigated a successful application of the proposed method to neural networks with sinusoidal feature mappings in \ref{sec4.3}. To summarize, we believe that our AL-PINNs form a universal framework that can be successfully applied to a variety of PDEs.

\section{Authors' Contribution}
Hwijae Son and Sung Woong Cho conceived the presented idea, developed the theory, and performed the simulations. Hyung Ju Hwang verified the analytical methods and supervised the findings of this work. All authors discussed the results and contributed to the final manuscript.

\section{Acknowledgement} Hwijae Son was supported by the National Research Foundation of Korea (NRF) grant funded by the Korea government (MSIT) (No. NRF-2022R1F1A1073732) and the research fund of Hanbat National University in 2022. Hyung Ju Hwang was supported by the National Research Foundation of Korea (NRF) Grants (RS-2022-00165268), by the Institute for the Information and Communications Technology Promotion (IITP) Grant through Korean Government [Artificial Intelligence Graduate School Program (Pohang University of Science and Technology (POSTECH))] (2019-0-01906), and by the Information Technology Research Center (ITRC) Support Program (IITP-2020-2018-0-01441).

\appendix

\section{PDEs and Loss Functions} \label{appendix_pde}
\subsection{Helmholtz equation}
The 2-D Helmholtz equation reads: \begin{equation}
\begin{aligned}\label{Helmholtz}
    &\Delta u + u = f(x,y), &&\text{ for } (x, y) \in \Omega, \\
    &u(x,y) = 0, &&\text{ for } (x,y) \in \partial\Omega, 
\end{aligned}
\end{equation}
where $\Omega = [-1,1]\times[-1,1]$. If we take 
\begin{align}
    f(x,y) &= -\pi^2 \sin(\pi x)\sin(4\pi y) \nonumber \\
    &- (4\pi)^2 \sin(\pi x)\sin(4\pi y) + \sin(\pi x)\sin(4\pi y), \nonumber
\end{align}
then one can readily show that $u(x,y)=\sin(\pi x)\sin(4\pi y)$ is an analytic solution. The proposed loss function reads as:
\begin{equation}
    \begin{aligned}\label{Helmholtz_ALM}
    \mathcal{L}_{\lambda_n, \beta}^{(A)} &\approx \frac{1}{N_r}\sum_{i=1}^{N_r}(\Delta u_{nn}(x_i) + u(x_i) - f(x_i))^2 \\
    &+ \frac{\beta}{N_B}\sum_{j=1}^{N_b} u_{nn}^2(x_j) \\
    &+ \frac{1}{N_B} \sum_{j=1}^{N_B}\lambda_n(x_j)u_{nn}(x_j),
    \end{aligned}
\end{equation}

\subsection{Viscous Burgers equation}
We consider the viscous Burgers equation: 
\begin{equation}
    \begin{aligned}\label{burgers_equation}
        &\partial_{t} u + \partial_x(\frac{1}{2}u^2 - c\partial_xu) = 0, &&\text{ for } (t,x)\in[0,1]\times[-1,1], \\
        &u(0,x) = -\sin(\pi x), &&\text{ for } x\in[-1,1], \\
        &u(t, -1) = u(t, 1) = 0, &&\text{ for } t\in[0,1],
    \end{aligned}
\end{equation}
for $(t,x) \in [0,1]\times[-1,1]$, and $c=\frac{0.01}{\pi}$. This setting is the same as in \cite{mcclenny2020self,kim2021dpm}. The viscous Burgers' equation admits an analytic solution as presented in \cite{basdevant1986spectral}. In this case, the proposed loss function reads as:

\begin{equation}
\begin{aligned}\label{Burgers_ALM}
        \mathcal{L}_{\lambda_n, \beta}(\theta) &\approx \frac{1}{N_r} \sum_{i=1}^{N_r} (\partial_t u_{nn}(t_i, x_i) +\partial_x( \frac{1}{2}u_{nn}^2 - c \partial_x u_{nn})(t_i,x_i)^2 \\
        &+ \frac{\beta}{N_I}\sum_{j=1}^{N_I} (u_{nn}(0,x_j) + \sin(\pi x_j))^2 \\
        &+ \frac{1}{N_I} \sum_{j=1}^{N_I} (\lambda^{(1)}_n (x_j))(u_{nn}(0,x_j) + \sin(\pi x_j))\\ 
        &+ \frac{\beta}{N_B}\sum_{k=1}^{N_B} (u_{nn}(t_k, -1) - u_{nn}(t_k,1))^2 \\
        &+ \frac{1}{N_B}\sum_{k=1}^{N_B} (\lambda^{(2)}_n(t_k))(u_{nn}(t_k, -1) - u_{nn}(t_k,1)).
\end{aligned}
\end{equation}

\subsection{Klein--Gordon equation} 
The Klein--Gordon equation we consider reads:
\begin{equation}
\begin{aligned}\label{Klein_Gordon_equation}
    &\partial_{t}^2u - \partial_{x}^2u + u^3 = f(t, x), &&\text{ for } (t,x)\in[0,1]\times[0,1], \\
    &u(0, x) = g_1(x), &&\text{ for } x\in[0,1], \\
    &\partial_{t} u(0, x) = g_2(x), &&\text{ for } x\in[0,1], \\
    &u(t,x) = h(t,x), &&\text{ for } (t,x)\in[0,1]\times\{0,1\},
\end{aligned}
\end{equation}
where $f, g_1, g_2, h$ are computed using a fabricated solution 
\begin{equation}
    u(t,x) = x\cos(5\pi t) + (tx)^3, \nonumber 
\end{equation}
as in \cite{wang2021understanding}. In this example, the proposed loss function reads as: 
\begin{equation}
\begin{aligned}\label{Klein-Gordon_ALM}
        \mathcal{L}_{\lambda_n, \beta}(\theta) &\approx \frac{1}{N_r} \sum_{i=1}^{N_r}(\partial^2_t u_{nn}(x_i) - \partial^2_x u_{nn}(x_i) + u_{nn}^3(x_i) - f(x_i))^2 \\
        &+ \frac{\beta}{N_I}\sum_{j=1}^{N_I} (u_{nn}(0,x_j)-g_1(x_j))^2 \\
        &+ \frac{1}{N_I}\sum_{j=1}^{N_I} (\lambda^{(1)}_n(x_j)) (u_{nn}(0,x_j)-g_1(x_j))\\
        &+ \frac{\beta}{N_I}\sum_{j=1}^{N_I} (\partial_t u_{nn}(0,x_j)-g_2(x_j))^2\\
        &+ \frac{1}{N_I}\sum_{j=1}^{N_I} (\lambda^{(2)}_n(x_j))(\partial_t u_{nn}(0,x_j)-g_2(x_j)) \\ 
        &+ \frac{\beta}{N_B}\sum_{k=1}^{N_B} (u_nn(t_k,x_k) - h(t_k,x_k))^2 \\
        &+ \frac{\beta}{N_B}\sum_{k=1}^{N_B} (\lambda^{(3)}_n(t_k,x_k)) (u_{nn}(t_k,x_k)-h(t_k,x_k)).
\end{aligned} 
\end{equation}

\subsection{2D transient Navier-Stokes equation} 
The 2D transient incompressible Navier-Stokes equation is given by :
\begin{equation}
    \begin{aligned}\label{Navier_Stokes_equation}
        &\partial_x u + \partial_y v = 0, \\
        &\partial_t u + u \partial_x u + v \partial_y u = -\partial_x p + \nu (\partial_{xx}u + \partial_{yy}u), \\
        &\partial_t v + u \partial_x v + v \partial_y v = -\partial_y p + \nu (\partial_{xx}v + \partial_{yy}v), \\
    \end{aligned}
\end{equation} where $u, v$  are the velocity and $p$ is the pressure. We assume the Dirichlet boundary conditions for $u$ and $v$ and the Neumann boundary condition for $p$. 

Here, the exact closed form solution is given by :
\begin{equation}
    \begin{aligned}
    &u(t,x,y) = \sin(\pi x) \cos(\pi y) \exp(-2\pi^2 \nu t), \\
    &v(t,x,y) = -\cos(\pi x) \sin(\pi y) \exp(-2\pi^2 \nu t), \\
    &p(t,x,y) = \frac{1}{4} (\cos(2\pi x) + \sin(2\pi y)) \exp(-4 \pi^2 \nu t).
    \end{aligned}
\end{equation}

We solve the problem for $(t,x,y) \in [0,2] \times [0.5, 4.5] \times [0.5, 4.5]$. We impose the initial conditions using $u(0,x,y), v(0,x,y),$ and $p(0,x,y)$. the proposed loss function is defined in the same manner: 
\footnotesize
\begin{equation}
\begin{aligned}\label{Navier-Stokes_ALM}
        \mathcal{L}_{\lambda_n, \beta}(\theta) 
        &\approx \frac{1}{N_r} \sum_{i=1}^{N_r} (\partial_x \hat{u}_i - \partial_y \hat{v}_i)^2\\
        &+ \frac{1}{N_r} \sum_{i=1}^{N_r} (\partial_t \hat{u}_i + \hat{u}_i \partial_x \hat{u}_i + \hat{v}_i \partial_y \hat{u}_i + \partial_x p_i - \nu(\partial_{xx}\hat{u}_i + \partial_{yy}\hat{u}_i))^2 \\
        &+ \frac{1}{N_r} \sum_{i=1}^{N_r} (\partial_t \hat{v}_i + u_i \partial_x \hat{v}_i + \hat{v}_i \partial_y \hat{v}_i + \partial_y p_i - \nu(\partial_{xx}\hat{v}_i + \partial_{yy}\hat{v}_i))^2 \\
        &+ \frac{\beta}{N_I}\sum_{j=1}^{N_I} (\hat{u}_j-u(0,x_j,y_j))^2 + \frac{1}{N_I}\sum_{j=1}^{N_I} (\lambda^{(1)}_n(x_j,y_j)) (\hat{u}_j-u(0,x_j,y_j))\\
        &+ \frac{\beta}{N_I}\sum_{j=1}^{N_I} (\hat{v}_j-v(0,x_j,y_j))^2 + \frac{1}{N_I}\sum_{j=1}^{N_I} (\lambda^{(2)}_n(x_j,y_j)) (\hat{v}_j-v(0,x_j,y_j))\\
        &+ \frac{\beta}{N_I}\sum_{j=1}^{N_I} (\hat{p}_j-p(0,x_j,y_j))^2 + \frac{1}{N_I}\sum_{j=1}^{N_I} (\lambda^{(3)}_n(x_j,y_j)) (\hat{p}_j-p(0,x_j,y_j))\\
        &+ \frac{\beta}{N_B}\sum_{k=1}^{N_B} (\hat{u}_k - u(t_k,x_k,y_k))^2 + \frac{1}{N_B}\sum_{k=1}^{N_B} (\lambda^{(4)}_n(t_k, x_k, y_k)) (\hat{u}_k-u(t_k,x_k,y_k))\\
        &+ \frac{\beta}{N_B}\sum_{k=1}^{N_B} (\hat{v}_k - v(t_k,x_k,y_k))^2 + \frac{1}{N_B}\sum_{k=1}^{N_B} (\lambda^{(5)}_n(t_k, x_k, y_k)) (\hat{v}_k-u(t_k,x_k,y_k))\\
        &+ \frac{\beta}{N_B}\sum_{k=1}^{N_B} (\partial_n\hat{p}_k - \partial_n p(t_k,x_k,y_k))^2 \\
        &+ \frac{1}{N_B}\sum_{k=1}^{N_B} (\lambda^{(6)}_n(t_k, x_k, y_k)) (\partial_n\hat{p}_k-\partial_n p(t_k,x_k,y_k)),
\end{aligned} 
\end{equation}
\normalsize
where $\hat{u}, \hat{v},$ and $\hat{p}$ are the corresponding neural network functions.

\section{Hyperparameters}\label{appendix_hyper}
For all algorithm-architecture pairs in the tables in Section \ref{sec4}, we test the learning rate $\eta_{\theta}$ from $\{10^{-3},10^{-4},10^{-5}\}$. We choose $\beta$ from $\{10^{0},10^{1},5\times10^{1},10^{2},5\times10^{2},10^{3},10^{4}\}$ and $\eta_{\lambda}$ from $\{10^{0},10^{-1},10^{-2},10^{-3},10^{-4}\}$ and report the best error. We set the initial value of the multiplier to be $\lambda_0=\{0,0,...,0\}$ for the proposed AL-PINNs. Table \ref{best_hyper} shows the best hyperparameters for each equation.

\renewcommand{\arraystretch}{1.5}
\begin{table}[h]
\centering

\caption{Best hyperparameter configurations.}
\label{best_hyper}
\vskip 0.1in
\resizebox{1\columnwidth}{!}{
\begin{tabular}{|c|ccc|ccc|ccc|}
\hline
\multirow{2}{*}{Best Hyperparameters} & \multicolumn{3}{c|}{Helmholtz equation}                                                & \multicolumn{3}{c|}{Viscous Burgers equation}                                                  & \multicolumn{3}{c|}{Klein--Gordon equation}                                            \\ \cline{2-10} 
                  & \multicolumn{1}{c|}{$\beta$} & \multicolumn{1}{c|}{$\eta_{\lambda}$} & $\eta_{\theta}$ & \multicolumn{1}{c|}{$\beta$} & \multicolumn{1}{c|}{$\eta_{\lambda}$} & $\eta_{\theta}$ & \multicolumn{1}{c|}{$\beta$} & \multicolumn{1}{c|}{$\eta_{\lambda}$} & $\eta_{\theta}$ \\ \hline
M1                & \multicolumn{1}{c|}{$10^3$}  & \multicolumn{1}{c|}{1}                & $10^{-3}$       & \multicolumn{1}{c|}{1}       & \multicolumn{1}{c|}{$10^{-4}$}        & $10^{-4}$       & \multicolumn{1}{c|}{$5\times10^{2}$}  & \multicolumn{1}{c|}{1}                & $10^{-3}$       \\ \hline
M2                & \multicolumn{1}{c|}{$5\times10^2$}  & \multicolumn{1}{c|}{1}                & $10^{-4}$       & \multicolumn{1}{c|}{1}       & \multicolumn{1}{c|}{$10^{-3}$}        & $10^{-4}$       & \multicolumn{1}{c|}{$5\times10^{2}$}  & \multicolumn{1}{c|}{1}                & $10^{-3}$       \\ \hline
M3                & \multicolumn{1}{c|}{$10^3$}  & \multicolumn{1}{c|}{1}                & $10^{-4}$       & \multicolumn{1}{c|}{1}       & \multicolumn{1}{c|}{$10^{-3}$}        & $10^{-4}$       & \multicolumn{1}{c|}{$5\times10^{2}$}  & \multicolumn{1}{c|}{1}                & $10^{-3}$       \\ \hline
M4                & \multicolumn{1}{c|}{$5\times10^2$}  & \multicolumn{1}{c|}{1}                & $10^{-3}$       & \multicolumn{1}{c|}{1}       & \multicolumn{1}{c|}{$10^{-3}$}        & $10^{-3}$       & \multicolumn{1}{c|}{$5\times10^{2}$}  & \multicolumn{1}{c|}{1}                & $10^{-3}$       \\ \hline
\end{tabular}}
\end{table}
\renewcommand{\arraystretch}{1}

For the baseline algorithms, we followed the best hyperparameter settings provided in each paper. For example, we set the initial $\lambda$ to be $(1,1,...,1)$ for SA method in \cite{mcclenny2020self}. For the LRA method, we set $\alpha=0.1$ and the initial $\lambda$ to be ones, as in \cite{wang2021understanding}.

\section{Computational Cost}\label{appendix_Computational_Cost}
We report the actual computing time for one epoch for different models and equations. We use NVIDIA GeForce RTX 3090 for the measurement. 

Quantitatively, if we adopt a uniform mesh($100\times100 \text{ in } [0, T] \times \Omega$) for the Burgers equations, AL-PINNs require 200 additional learnable parameters, the Soft Attention(SA) algorithm requires 10,200 additional learnable parameters, and the Learning Rate Annealing(LRA) algorithm requires no additional parameters. Thus, in terms of the number of additional learnable parameters, AL-PINNs are more efficient than the SA algorithm and less efficient than the LRA algorithm. The additional learnable parameters are the key factors of computational cost. However, as we can see in Table \ref{table_10000epoch}, there is no significant difference in the computing time per epoch among the algorithms. Therefore, we conclude that the increment of the computational complexity of AL-PINNs is negligible.

\begin{table}[h]\resizebox{1\columnwidth}{!}{
\begin{tabular}{|c|ccccc|}
\hline
\multirow{2}{*}{Computing time (ms)} & \multicolumn{5}{c|}{Helmholtz equation}                                                                                          \\ \cline{2-6} 
                                    & \multicolumn{1}{c|}{PINNs}  & \multicolumn{1}{c|}{SA}     & \multicolumn{1}{c|}{LRA}    & \multicolumn{1}{c|}{NTK}    & AL-PINNs \\ \hline
M1                                  & \multicolumn{1}{c|}{39.9} & \multicolumn{1}{c|}{40.3} & \multicolumn{1}{c|}{50.2} & \multicolumn{1}{c|}{N/A} & 40.5   \\ \hline
M2                                  & \multicolumn{1}{c|}{23.1} & \multicolumn{1}{c|}{23.1} & \multicolumn{1}{c|}{30.6} & \multicolumn{1}{c|}{N/A} & 23.7   \\ \hline
M3                                  & \multicolumn{1}{c|}{41.1} & \multicolumn{1}{c|}{41.1} & \multicolumn{1}{c|}{50.6} & \multicolumn{1}{c|}{N/A} & 41.4   \\ \hline
M4                                  & \multicolumn{1}{c|}{23.2} & \multicolumn{1}{c|}{23.5} & \multicolumn{1}{c|}{30.6} & \multicolumn{1}{c|}{N/A} & 23.7   \\ \hline
\end{tabular}}
\resizebox{1\columnwidth}{!}{
\begin{tabular}{|c|ccccc|}
\hline
\multirow{2}{*}{Computing time (ms)} & \multicolumn{5}{c|}{Viscous Burgers Equation}                                                                                    \\ \cline{2-6} 
                  & \multicolumn{1}{c|}{PINNs}  & \multicolumn{1}{c|}{SA}     & \multicolumn{1}{c|}{LRA}    & \multicolumn{1}{c|}{NTK}    & AL-PINNs \\ \hline
M1                & \multicolumn{1}{c|}{34.4} & \multicolumn{1}{c|}{34.5} & \multicolumn{1}{c|}{35.0} & \multicolumn{1}{c|}{N/A} & 35.8   \\ \hline
M2                & \multicolumn{1}{c|}{17.1} & \multicolumn{1}{c|}{17.6} & \multicolumn{1}{c|}{18.0} & \multicolumn{1}{c|}{N/A} & 18.4   \\ \hline
M3                & \multicolumn{1}{c|}{37.2} & \multicolumn{1}{c|}{34.9} & \multicolumn{1}{c|}{38.0} & \multicolumn{1}{c|}{N/A} & 38.7   \\ \hline
M4                & \multicolumn{1}{c|}{17.3} & \multicolumn{1}{c|}{17.7} & \multicolumn{1}{c|}{17.8} & \multicolumn{1}{c|}{N/A} & 18.6   \\ \hline
\end{tabular}}
\resizebox{1\columnwidth}{!}{
\begin{tabular}{|c|ccccc|}
\hline
\multirow{2}{*}{Computing time (ms)} & \multicolumn{5}{c|}{Klein Gordon Equation}                                                                                       \\ \cline{2-6} 
                  & \multicolumn{1}{c|}{PINNs}  & \multicolumn{1}{c|}{SA}     & \multicolumn{1}{c|}{LRA}    & \multicolumn{1}{c|}{NTK}    & AL-PINNs \\ \hline
M1                & \multicolumn{1}{c|}{45.2} & \multicolumn{1}{c|}{47.0} & \multicolumn{1}{c|}{61.2} & \multicolumn{1}{c|}{N/A} & 46.0   \\ \hline
M2                & \multicolumn{1}{c|}{28.5} & \multicolumn{1}{c|}{29.0} & \multicolumn{1}{c|}{40.4} & \multicolumn{1}{c|}{N/A} & 29.6   \\ \hline
M3                & \multicolumn{1}{c|}{46.8} & \multicolumn{1}{c|}{48.0} & \multicolumn{1}{c|}{61.0} & \multicolumn{1}{c|}{N/A} & 48.1   \\ \hline
M4                & \multicolumn{1}{c|}{28.7} & \multicolumn{1}{c|}{29.3} & \multicolumn{1}{c|}{40.4} & \multicolumn{1}{c|}{N/A} & 30.1   \\ \hline
\end{tabular}} \caption{Average computing time in milliseconds (ms) for one training epoch.}
\label{table_10000epoch}
\end{table}

\section{Proof of Theorems}\label{appendix_proof}
\Conv*

\begin{proof}
    If $r > \inf F$, then $x_k$ is contained in the bounded set $\bigcup_{n\in\mathbb{N}} \{x\in X : F_n(x) \leq r\}$ for a sufficiently large $k$ by the inequality 
    \begin{align*}
        \inf_{x\in X} F \ge \limsup_{n \to \infty} (\inf_{x\in X} F_n)= \limsup_{n \to \infty}(\inf_{x\in X} F_n + \delta_n).
    \end{align*}
    The reflexivity of $X$ ensures that $\left\{ x_n \right\}_{n\in\mathbb{N}}$ has a weakly convergent subsequence $\left\{x_{n_k}\right\}_{k\in \mathbb{N}}$. And the limit point $x$ should be a unique minimizer of $F$ by Lemma \ref{quasi}. Suppose that there exists a subsequence $\left\{x_{n_{m}}\right\}_{m\in\mathbb{N}}$ and an element $G$ in $X^{*}$ such that $\{G(x_{n_{m}})\}_{m \in \mathbb{N}}$ is not converging to $G(x)$. ${x_{n_{m}}}$ attains a further subsequence which converges to $x$ by the reflexivity of $F$ so that a contradiction arises. Therefore, $x_n$ converges to $x$ weakly. Along with the Admissible limit point property, we conclude that $x$ lies in $F_A$ so that $F$ indeed attains its minimum at $x$.
\end{proof}

\quasi*

\begin{proof}
    First, we consider a recovery sequence $\left\{y_n \right\}_{n\in\mathbb{N}}$ for $y\in F_A$ to establish the following inequalities.
    \begin{align*}
        F(y)= \lim_{n\rightarrow\infty} F_n(y_n) = \limsup_{n\to \infty} F_n(y_n) \ge \limsup_{n\to \infty} (\inf_X F_n).
    \end{align*}
    Taking the infimum over $y$, we reformulate the inequality as $\inf_X F\ge \limsup_{n\to \infty} (\inf_x F_n)$. By the admissible limit point property, $x \in F_A$ and therefore, the liminf inequality yields the following. 
    \begin{align*}
        F(x) &\le \liminf_{n\to \infty} F_n(y_n) \le \limsup_{n\to \infty} (F_n (y_n)) \\&\le \limsup_{n\to \infty}(\inf_X F_n+\delta_n)=\limsup_{n\to \infty}(\inf_X F_n)\\&\le \inf_X(F).
    \end{align*}
    Consequently, $x$ should be a minimizer of $F$.
\end{proof}

\losszero*
\begin{proof}
Let $\epsilon > 0$ be given and let $u$ denote a strong solution of the Helmholtz equation \eqref{Helmholtz_eq}. By theorem \ref{universal}, there exists $u_{nn}(x,y)$ defined as in theorem \ref{universal} such that 
\begin{equation}\max_{|\alpha|\leq2} \|D^{\alpha}(u) - D^{\alpha}(u_{nn}(x,y))\|_{L^{\infty}(\bar\Omega)} < \epsilon.\nonumber
\end{equation}
Then, 
\begin{equation}
    \begin{aligned}
        &L_n(u_{nn}) \\&= \|\Delta u_{nn}(x,y) + k^2 u_{nn}(x,y) - q(x,y)\|_{L^2(\Omega)}\\
         &+ \beta\|u_{nn}(x) - g(x)\|_{L^2(\partial\Omega)} + \langle\lambda_n, u_{nn}(x)-g(x)\rangle_{L^2(\partial\Omega)} \\
         &= \| \Delta(u_{nn}(x,y)-u(x,y)) + k^2 (u_{nn}(x,y)-u(x,y))\|_{L^2(\Omega)}\\
         &+ \beta\|u_{nn}(x,y) - u(x,y)\|_{L^2(\partial\Omega)} + \langle\lambda_n, u_{nn}(x)-g(x)\rangle_{L^2(\partial\Omega)} \\
         &\leq \|\Delta(u_{nn}(x,y)-u(x,y))\|_{L^2(\Omega)} + \| k^2 (u_{nn}(x,y)-u(x,y))\|_{L^2(\Omega)} \\ &+\beta\|u_{nn}(x,y) - u(x,y)\|_{L^2(\partial\Omega)} + \|\lambda_n\|_{L^2(\partial\Omega)}\|u_{nn}-g\|_{L^2(\partial\Omega)}\\ &\leq C\epsilon
    \end{aligned}
\end{equation}
for some $C>0$, provided that $\|\lambda_n\|_{L^2(\partial\Omega)}<\infty$. This completes the proof.
\end{proof}

\equi*
\textbf{Proof of Theorem \ref{equi} for the Helmholtz equation}\\
The Helmholtz equation reads: 
\begin{equation}\label{Helmholtz_eq}
    \begin{aligned}
        &N(u):=\Delta u(x,y)+k^2 u(x,y)=f(x,y), &&\quad (x, y)\in {\Omega}
        \\&u=g,  &&\quad (x,y)\in {\partial \Omega}
    \end{aligned}
\end{equation}
for a constant $k$. In this paper, we only consider the case where $\Omega$ is a rectangle $[a, b]\times [c, d]$ and there exists $h\in H^{2}(\Omega)$ such that $T(h)=g$ for the trace operator $T$. There are existence and uniqueness theorems for linear elliptic equations. Let us denote the set of all eigenvalues of the Laplace operator $L(u):=-\Delta u $ by $\Sigma$ (i.e. if $k^2 \notin \Sigma$, then $L(u):=-\Delta u - k^2 u = f$ has a unique solution with $u=0$ on $\partial \Omega$). Then, the set $\Sigma$ is at most countable (see theorem 5 in Section 6.2 of \cite{evans1998partial}) and the following theorem holds.
\begin{theorem}[Thm 6 in Sec 6.2 of \cite{evans1998partial}] \label{Regularity_Spectrum}
    Let $f\in L^2(\Omega)$. Suppose $k^2 \notin \Sigma$ and let $u\in{H_{0}^{1}(\Omega)}$ be the unique solution of the following equation.
    \begin{equation} \label{eigenvalue}
        \begin{aligned}
                           & \Delta u +k^2 u= f, & \text{in} &&\Omega,
                \\ &u=0, &\text{on} &&\partial \Omega.        
        \end{aligned}
    \end{equation}
    Then, there exist a constant $C$ which depends on $k^2$ and $\Omega$ such that the following inequality holds. 
    \begin{align*}
        ||u||_{L^2(\Omega)} \le C(||f||_{L^2(\Omega)}).
    \end{align*}
\end{theorem}
Note that the constant $C$ does not depend on $f$. Since $f$ is a component of the loss function, the above theorem implies the boundedness of $||u||_{L^2(\Omega)}$ under the exact boundary condition. Now we refer the theorem about higher regularity of the solution. That is, the solution of \eqref{eigenvalue} indeed lies in $H^2(\Omega)$. 
\begin{theorem}[Thm 4.3.1.4 in \cite{grisvard2011elliptic}] \label {Helmholtz_Estimate} Suppose that $\Omega$ is bounded open and its boundary $\partial \Omega$ is the union of finite line segments. Then, we can have the following estimate of $\|u\|_{H^2(\Omega)}$ when $u\in H^{1}(\Omega)$ is the solution of  \eqref{eigenvalue}. 
    \begin{align*}
        \| u\|_{H^{2}(\Omega)} \le C(\| f \|_{L^2(\Omega)} +\|u\|_{L^2(\Omega)}),
    \end{align*}
where $C$ is a constant depending on $k^2$ and $\Omega$.
\end{theorem}

Suppose that the neural network $w$ satisfies the following. 
    \begin{equation} 
        \begin{aligned}
                           & \Delta w +k^2 w= f^{*}, & \text{in} &&\Omega,
                \\ &w=g^{*}, &\text{on} &&\partial \Omega.        
        \end{aligned}
    \end{equation}    
Under the whole training process, we assume that the uniform boundedness of $\|\frac{d}{dy}(g^{*}(a, y))\|_{L^2(\{a\}\times[b, d])}$, $\|\frac{d}{dy}(g^{*}(c, y))\|_{L^2(\{c\}\times[b,d])}$ and $\|\frac{d}{dx}(g^{*}(x, b))\|_{L^2([a,c]\times \{b\})},\\ \|\frac{d}{dx}(g^*(x, d))\|_{L^2([a,c]\times \{d\})}$ for $x\in (a,c)$. That is, there exists a constant $\epsilon>0$ such that
\begin{align*}
    \|\frac{d^{2}}{dy^2}g^{*}(a, y)\|_{L^2(\{a\}\times[b, d])}+ \|\frac{d^2}{dy^2}g^{*}(c, y)\|_{L^2(\{c\}\times[b,d])} +\\ \|\frac{d^2}{dx^2}g^{*}(x, b)\|_{L^2([a,c]\times \{b\})}+ \|\frac{d^2}{dx^2}g^*(x, d)\|_{L^2([a,c]\times \{d\})} \leq \epsilon
\end{align*}
Note that the assumption implies the fact that $\|g^{*}(a, y)\|_{L^{\infty}([a,c]\times \{b\})}$ is uniformly bounded during the training process. To explain it more precisely, Morrey's inequality yields that 
\begin{align*}
    &\|g^{*}(a, y)\|_{L^{\infty}([a,c]\times \{b\})} \leq C\|g^{*}(a, y)\|_{H^1([a,c]\times \{b\})}, \\& \|d/dy(g^{*}(a, y))\|_{L^\infty([a,c]\times \{b\})} \leq C\|d/dy(g^{*}(a, y))\|_{H^1([a,c]\times \{b\})} .
\end{align*}
with the fact that $g^{*}(a, y)$ should be a smooth function since it is consistent with the smooth neural network $w$ on the boundary. Therefore, the desired property of $\|g^{*}(a, y)\|_{L^{\infty}([a,c]\times \{b\})}$ is valid by the uniform boundedness of $\|g^{*}(a, y)\|_{L^2([a,c]\times\{b\})}$ which will be proved below.
Now we are ready to prove the main theorem \ref{equi} for the Helmholtz equation.

\begin{proof}
    We consider the case $k^2\notin \Sigma$ where $\Sigma$ denotes the spectrum of the Laplace operator. By assumption, there exists a harmonic polynomial $h_{1}(x, y)=a_{0}(x^2 -y^2)+a_{1}xy+a_2 x+a_3 y+a_4 \in H^{2}(\Omega)$ which satisfies $h=g_{1}^{*}$ on the four vertices $[a, c], [a, d], [b, c]$ and $[b,d]$. And  $h_2(x,y)$ defined as below lies in  $H^{2}(\Omega)$ with a uniformly bounded norm during the entire training process. 
    \begin{align*}
    h_2(x, y)=& \frac{x-c}{a-c}\cdot(g-h_1)(a, y)+\frac{x-a}{c-a}(g-h_1)(c, y)+\\& \frac{y-d}{b-d}(g-h_1)(x, b)+\frac{y-b}{d-b}(g-h_2)(x, d)\in H^2(\Omega). 
    \end{align*}
    Now we estimate the difference $e(:=w-h=w-(h_1+h_2)).$ First of all, $e$ should satisfy the following equation.
    \begin{equation*}
        \begin{aligned}
                           & \Delta e +k^2 e= f^{*}-\Delta h-k^2 h, & \text{in} &&\Omega,
                \\ &e=0, &\text{on} &&\partial \Omega.        
        \end{aligned}
    \end{equation*}
    Using Theorem \ref{Regularity_Spectrum}, we get the estimate for $\|e\|_{L^2(\Omega)}$ where $C$ is a constant which depends only on $k^2$ and $\Omega$. 
    \begin{align*}
        \|e\|_{L^{2}(\Omega)}\le & C\|f^{*}-\Delta h-k^2 h\|_{L^{2}(\Omega)} \\\leq& C(\|f^{*}\|_{L^2(\Omega)}+(1+k^2)\|h\|_{H^2(\Omega)})  
    \end{align*}
    Now we can conclude that $w\in H^2(\Omega)$ with theorem \ref{Helmholtz_Estimate}.
    \begin{align*}
        \|w\|_{H^{2}(\Omega)} \leq& \|e\|_{H^2(\Omega)} +\|h\|_{H^2(\Omega)}\\ 
        \leq& C(k^2, \Omega) (\|e\|_{L^2(\Omega)}+ \|f^{*}-\Delta h-k^2 h\|_{L^2(\Omega)})+\|h\|_{H^2(\Omega)}\\
        \leq& C(k^2, \Omega)(C+1)\|f^*\|_{L^2(\Omega)} +((1+k^2)CC(k^2, \Omega)+\\&(1+k^2) C(k^2 , \Omega)+1)\|h\|_{H^2(\Omega)}
    \end{align*}
    where $C(k^2, \Omega)$ denotes the constant for theorem \ref{Helmholtz_Estimate}. Now, suppose that $L_n(w) \leq r$ for some $r\in\mathbb{R}$ and $n\in \mathbb{N}$. That is, 
    \begin{align*}
    L_n(w)&=\|Nw-f\|_{L^2(\Omega)}^2 + \beta \|Tw-g\|_{L^2{(\partial\Omega})}^2+ \langle \lambda_n ,  Tw-g\rangle_{L^2{(\partial \Omega)}}\\
        &=\|f^{*}-f\|_{L^2(\Omega)}+\beta\|g^{*}-g\|^{2}_{L^2(\Omega)}+\langle \lambda_n ,  g^*-g\rangle_{L^2{(\partial \Omega)}}\\
        &\leq r.
    \end{align*}
    By Lemma \ref{Bound} which covers the uniform boundedness of $\|\lambda_{n}\|_{L^2(\Omega)}$ during the whole training process, the H\"olders inequality $\langle \lambda_n, Tw - g\rangle_{L^2(\partial\Omega)} \ge -\|\lambda_n\|_{L^2(\partial \Omega)}\|Tw - g\|_{L^2(\partial\Omega)}$ implies the fact that  $\|g^*-g\|_{L^2(\partial \Omega)}$ is bounded by a constant depending on $r$. To explain it more precisely, note that $L_{n}(w)$ is bounded by the quadratic equation for $\|g^*-g\|_{L^2(\partial \Omega)}$, and therefore $\|g^{*}-g\|_{L^2(\Omega)}$ cannot be large indefinitely for a fixed constant $r$. Thus, $\|f^{*}\|_{L^2(\Omega)}$ and $\|g^*\|_{L^2(\partial\Omega)}$ are also bounded by a constant that depends on $r$ for a fixed $f$ and $g$ through standard argument using the triangular inequality. In conclusion, $\{w|L_{n}{w} \le r $ for some $n\in \mathbb{N}\}$ should be contained in a bounded set in $H^{2}(\Omega)$.
\end{proof}
\textbf{Proof of Theorem \ref{equi} for the viscous Burgers equation}\\
The viscous Burgers equation reads: 
\begin{equation}\label{Burgers_eq}
    \begin{aligned}
        &\partial_{t} u +\frac{1}{2}\partial_{x}u^2-\nu \partial_{x}^2u=0, && (t,x) \in [0, T]\times\Omega,\\
        &u(0, x)=u_0 (x), && x\in \Omega,\\
        &u(t, a)=u(t, b)=0, && t \in [0, T],\\
    \end{aligned}
\end{equation}
where $\Omega$ is an interval $[a, b]$ and $T$ is a fixed constant. Without loss of generality, set $a=0$ and $b=1$. Here, we note the solution space that is necessary to deal with the existence and uniqueness of the solution. The definition of anisotropic Sobolev space $H^{1,2}(R)$ is as follows.
\begin{align*}
    H^{1,2}(R)=\{u\in L^2(R)|\partial_t u \in L^2(R), \partial_x u \in L^2(R), \partial_x ^2 u \in L^2(R) \},
\end{align*}
where $R=\Omega \times [0, T]$ and $\|u\|_{L^2}(R)=\|u\|_{L^2([0, T];L^2(\Omega))}$ denotes \\ $(\int_{0}^{T}\|u(t, \cdot)\|_{L^2(\Omega)}^2 dt)^{1/2}$. The existence and uniqueness theorem for equation \eqref{Burgers_eq} are as follows.
\begin{theorem}(Thm 1.2 in \cite{benia2016existence})
    For the initial condition $u_{0}(x)$ which is contained in $H_0^1 (\Omega)$, the unique weak solution of equation \eqref{Burgers_eq} exists in $H^{1,2}(R)$.
\end{theorem}
We would like to introduce an assumption before we discuss the equi-coercivity. Let $w$ denote the smooth approximation of the solution $u$ such that
\begin{equation*}
    \begin{aligned}
    &\partial_{t}w+w\partial_{x}w-\nu \partial_{x}^2 w=f, && (t,x)\in [0, T]\times\Omega,
    \\&w(0, x)= g(x),  && x\in \Omega,   
    \\&w(t, a)=h_1(t), w(t, b)=h_2(t), && t\in [0, T].  
    \end{aligned}
\end{equation*} 
We assume the uniform boundedness of $\|\partial _t h_1(t)\|_{L^2([0, T])}$ and $\|\partial_t  h_2(t)\|_{L^2([0,T])}$ for the entire training process. That is, there exists a constant $\epsilon>0$ such that the boundary functions $g(x), h_1(t)$ and $h_2(t)$ satisfy the following.
\begin{align*}
    \|\frac{d}{dx} g(x)\|_{L^2([0, T])}, \|\frac{d}{dt} h_1(t)\|_{L^2([0,T])}, \|\frac{d}{dt} h_2(t)\|_{L^2([0,T])} \le \epsilon.
\end{align*}

Now, we begin the proof of \ref{equi} for the viscous Burgers equation.
\begin{proof}

    Now, we define the function $e(t,x)$ by $w(t,x)-I(t, x)$ when the interpolation function $I(t,x)\in H^{1,2}(R)$ denotes $h_1(t)\cdot(1-x) +h_2(t) \cdot x$. Then, $e(t, x)$ must be contained in $H^{1,2}(R)$ such that 
    \begin{equation}\label{Burgers_NeuralNet}
        \begin{aligned}
            &\partial_{t}e+e\partial_{x}e+e\partial_{x}I+I\partial_{x}e-\nu \partial_{x}^{2}e=f^*, && (t,x)\in [0, T]\times\Omega,
            \\&e(0, x)= g^*,  && x\in \Omega,   
            \\&e(t, a)= 0, e(t, b) = 0, && t\in [0, T],  
        \end{aligned}
    \end{equation}
    where $f^*=f-\partial_{t}I-I\partial_{x}I-\nu \partial_{x}^{2}I$ and $g^*=g(x) - g(0)(1-x)-g(1)x$. On the one hand, Morrey's inequality gives the inequality $g(0)=h_1(0)\leq \|h_1(t)\|_{L^{\infty}[0,T]} \leq C_{1} \|h_1(t) \|_{H^1([0,T])}$ so that $g^* \in H^1(\Omega)$ when $h_{1}(t) \in L^2([0, T])$.  By multiplying $e$ to the first equation and integrating over $\Omega$ with integration by parts, we derive the following equation. 
    \begin{align*}
        &\frac{1}{2} \frac{d}{dt}\|e(t, \cdot)\|^2_{L^2(\Omega)} + \nu \|e_x (t, \cdot)\|_{L^2(\Omega)}^{2}\\&=\int_{\Omega} f^{*} e dx +\nu \int_{\partial \Omega} e\partial_{x}e \cdot ndS - \int _{\Omega} e^2 \partial_{x} e dx - \int_{\Omega} e^2\partial_{x}I dx -\\&\int_{\Omega} Ie\partial_{x} edx.  
    \end{align*}
    By H\"older's inequality and scaled version of Cauchy's inequality with small $\eta>0$, we have
    \begin{align*}
        &\int_{\Omega} f^{*} e dx \\&\le \int_{\Omega} \frac{(f^{*})^2}{2} +\frac{ e^{2}}{2}dx = \frac{1}{2}(\|f^{*}(t, \cdot)\|^{2}_{L^2(\Omega)} +\|e(t, \cdot)\|^{2}_{L^2(\Omega)}).\\
        &\int_{\partial \Omega} e\partial_x e \cdot n dS = e(t, 1)\partial_{x}e(t,1)- e(t,0)\partial_{x} e(t, 0) = 0\\
        &\int_{\Omega}e^2 \partial_{x} e dx= \int_{\partial\Omega} \frac{1}{3} e^{3}\cdot n dS= 0\\
    \end{align*}
    \begin{align*}
        &\int_{\Omega}e^{2} \partial_{x} I dx \\&\le \int_{\Omega} \frac{e^{2}}{2}+\frac{ e^{2}(\partial_{x} I)^{2}}{2} dx \le\frac{\|e(t, \cdot)\|^{2}_{L^2(\Omega)}}{2}+\frac{\|e(t, \cdot)\|_{L^2(\Omega)}^{2}\|\partial_{x} I\|_{L^2(\Omega)}^2}{2}\\
        &\int_{\Omega} Ie\partial_{x} e dx \\&\le \int_{\Omega} \frac{I^2 e^2}{4\eta} + \eta (\partial_{x} e)^{2} dx \le \frac{1}{4\eta} \|I(t,\cdot)\|_{L^2(\Omega)}^{2}\|e\|_{L^2(\Omega)}^{2}+\eta \|\partial_{x} e\|^{2}_{L^2(\Omega)}
    \end{align*}

    Since $\eta>0$ is arbitrary small, we may assume that $2\eta<\nu$. Combining all estimates, we have the following upper bound of $d/dt(||w||_{L^2(\Omega)}^2)$.
    \begin{align} \label{Burgers_gronwall_1}
        &\frac{d}{dt} \|e\|_{L^2(\Omega)}^2 +\nu \|\partial_x e\|_{L^2(\Omega)}^{2} \\&\leq\nonumber  (2+\|\partial_{x} I\|_{L^2(\Omega)}^{2}+\frac{1}{2\eta}\|I\|_{L^2(\Omega)}^{2})\|e\|_{L^2(\Omega)}+\|f^{*}\|^{2}_{L^2(\Omega)}. 
    \end{align}
    Then, by the Gr\"onwall inequality, 
    \begin{align*}
        ||e||_{L^2([0,T]\times \Omega)}  &\leq  T\sup_{0\leq t \leq T}\|e\|_{L^2(\Omega)}^2 \\ &\leq T\exp(2T+\|\partial_{x}I\|_{L^2([0, T]\times \Omega)}+\frac{1}{2\eta}||I||^{2}_{L^2([0, T] \times \Omega)}) \cdot \\&\bigg( \|g^*\|_{L^2(\Omega)}^2 +||f^*||_{L^2([0, T] \times \Omega)} \bigg).
    \end{align*}
    for some constant $C$ depending only on $r, T$.
    Now we consider the case that the neural network $w$ whose image by $L_n(w)$ are bounded for some $n\in\mathbb{N}$. Suppose that 
    \begin{align*}
        L_{n}(w) =& L(w) + \langle \lambda_1(x), w(0, x)-u_0(x)\rangle_{L^2(\Omega)} + \\&  \langle \lambda_2(t), w(t,a)   \rangle_{L^2([0, T])} +\langle \lambda_3(t), w(t,b) \rangle_{L^2([0, T])}  \le r, 
    \end{align*}
    where
    \begin{align*}
        L(w) =& \|w_{t}+ww_{x}-\nu w_{xx}\|_{L^2([0,T]\times\Omega)}^2 + \|w(0, x)-u_0(x)\|_{L^2(\Omega)}^2 +\\& \|w(t,a)\|_{L^2([0,T])}^2 + \|w(t,b)\|_{L^2([0,T])}^2\\
        =& \|f\|_{L^2([0,T]\times\Omega)}^2 + \|g - u_0\|_{L^2(\Omega)}^2 + \|h_1\|_{L^2([0,T])}^2 + \|h_2\|_{L^2([0,T])}^2.
    \end{align*}
    for some $r > 0$. Since $\|\lambda_1(x)\|_{L^2(\Omega)}, \|\lambda_2(t)\|_{L^2([0, T])}$ and $\|\lambda_3(t)\|_{L^2([0, T])}$ are uniformly bounded during the training process as in \ref{Bound}, H\"older's inequality implies that $L(w)$ should be bounded by some constant $C$ which depends only on $r$. To explain this more precisely, H\"older's inequality yields 
    \begin{align*}
        &\langle \lambda_1(x), w(0, x)-u_0(x)\rangle_{L^2(\Omega)} + \langle \lambda_2(t), w(t,a)   \rangle_{L^2([0, T])} + \\&\langle \lambda_3(t), w(t,b) \rangle_{L^2([0, T])}
        \\&\geq -\|\lambda_{1}(x) \|_{L^2(\Omega)}\|w(0, x)- u_0(x)\|_{L^2(\Omega)}-\\&\|\lambda_2(t)\|_{L^2([0,T])}\|w(t,a)\|_{L^2([0, T])} -\| \lambda_{3}(t)\|_{L^2([0, T])} \|w(t, b)\|_{L^2([0, T])}.  
    \end{align*}
    Note that the last equation is of first order with respect to $\|g - u_0\|_{L^2(\Omega)}, \|h_1\|_{L^2([0,T])}$ and $\|h_2\|_{L^2([0,T])} $. Since $L(w)$ contains the second order terms, the three norms should be bounded by some constant which depends only on $r$ and the desired properties are followed. 
    \par Now we consider the composition of $w=(w-I)+I=e+I$. Then $\|e\|_{L^2([0, T]\times \Omega)}$ is bounded by a constant depending on $r$ due to the boundedness of three terms $\| g \|_{L^2(\Omega)}, \|h_1\|_{L^2([0,T])}$ and $\|h_2\|_{L^2([0,T])} $ with the above observation and Gr\"onwall's inequality. Note that $\| I \|_{L^2([0, T] \times \Omega)}$ has the same form of upper bound by construction. In conclusion, the usual triangular inequality gives the boundedness for $||w||_{L^2([0, T] \times \Omega)}$, and therefore, the set $M(r)=\{w\in L^2(\Omega)\cap| L_{n}(w) \leq r $ for some $ n\in\mathbb{N}\}$ is bounded in $L^2([0, T]\times \Omega)$ for $r>0$. 
    
    The set $M(r)$ is also bounded in $L^{2}([0,T]; H^1(\Omega))$ by the above estimate. If we integrate \eqref{Burgers_gronwall_1} over $[0, T]$ and considering the previous discussion about the boundedness of $\|e\|_{L^{2}(\Omega)}^{2}$, the term $\int_{0}^{T}\|\partial_{x} e\|_{L^2(\Omega)}^{2}$ should be bounded by a constant which depends only on $r$ when $w\in M(r)$.
    To derive further properties about the boundedness of $M(r)$, we multiply the equation (\ref{Burgers_NeuralNet}) by $-\partial_{x}^{2} e$. Using integration by parts, the following holds.
    \begin{align*}
        &\frac{1}{2} \frac{d}{dt}\|\partial_{x}e\|^{2}_{L^2(\Omega)}+\nu\|\partial_{x}^{2}e\|_{L^2(\Omega)}^{2}\\&=\nu \int_{\partial\Omega}\partial_{t} e\partial_{x} e\cdot n dS +\int_{\Omega} e\partial_{x}e\partial_{x}^{2}e dx +\int_{\Omega} e\partial_{x} I \partial_{x}^{2} e dx +\\&\int_{\Omega} I \partial_{x}e \partial_{x}^{2} e dx -\int_{\Omega} f^{*}\partial_{x}^2 edx.
    \end{align*}
    Using Morrey's inequality and H\"older's inequality, we can estimate the right side of the equation as follows. For an arbitrary small $\eta>0$, 
    \begin{align*}
        \int_{\partial\Omega} \partial_{t}e\partial_{x} e \cdot n dS &= \partial_{t}e(t, a)\partial_{x}e(t,a)-\partial_{t}e(t, b)\partial_{x}e(t,a)=0,
        \\ \int_{\Omega} e \partial_{x} e \partial_{x}^2 e dx &\leq  \frac{1}{4\eta}\|e\partial_{x}e\|_{L^2(\Omega)}^{2}+\eta\|\partial_{x}^2 e\|_{L^2(\Omega)}^{2} \\&\leq \frac{1}{4\eta} \|e\|^{2}_{L^{\infty}(\Omega)}\|\partial_{x}e\|_{L^2(\Omega)}^{2} +\eta\|\partial_{x}^{2}e\|_{L^2(\Omega)}^{2},
        \\ \int_{\Omega}e\partial_{x}I\partial_{x}^{2}e dx &\leq  \frac{1}{4\eta}\|e\partial_{x}I\|_{L^2(\Omega)}^{2}+\eta\|\partial_{x}^2 e\|_{L^2(\Omega)}^{2} \\&\leq \frac{1}{4\eta}\|e\|_{L^{\infty}(\Omega)}^{2}\|\partial_{x} I\|_{L^2(\Omega)}^{2}+\eta\|\partial_{x}^{2}e\|_{L^2(\Omega)}^{2},
        \\ \int_{\Omega}I\partial_{x}e\partial_{x}^{2}e dx &\leq \frac{1}{4\eta}\|I\partial_{x}e\|_{L^2(\Omega)}^{2}+\eta\|\partial_{x}^{2} e\|_{L^2(\Omega)}^2 \\&\leq \frac{1}{4\eta}\|I\|_{L^{\infty}(\Omega)}^{2}\|\partial_{x}e\|_{L^2(\Omega)}^2+\eta\|\partial_{x}^{2}e\|_{L^2(\Omega)}^{2},
        \\ \int_{\Omega}f^{**}\partial_{x}^{2}e dx &\leq \frac{1}{4\eta} \|f^{**}\|_{L^2(\Omega)}^{2} +\eta\|\partial_{x}^{2} e\|_{L^2(\Omega)}^{2},
        \\ \|e\|_{L^{\infty}(\Omega)}^{2} \le &  C_{1}\|e\|_{H^1(\Omega)}^{2}=  C_{1}\|\partial_{x} e\|_{L^2(\Omega)}^{2} \\ \|I\|_{L^{\infty}(\Omega)}^{2} \le & C_{1}\|I\|_{H^{1}(\Omega)}^{2}.
     \end{align*}
     Since $\eta>0$ is an arbitrary constant, we may assume that $\nu>8\eta$.  Consequently, we can derive the following inequality.
     \begin{align}\label{Burgers_Gronwall_H2}
         &\frac{d}{dt}\|\partial_{x} e\|^{2}_{L^2(\Omega)}+\nu \|\partial_{x}^{2}e\|_{L^2(\Omega)} \nonumber 
         \\&\leq (\frac{C_1}{2\eta}\|\partial_{x}e\|^{2}_{L^2(\Omega)}+\frac{C_{1}}{4\eta}\|I\|_{H^1(\Omega)}^{2})\|\partial_{x}e\|_{L^2(\Omega)}^{2} + \frac{1}{4\eta}\|f^{**}\|_{L^2(\Omega)}^{2}.
     \end{align}
     By the above discussion, the terms $\|\partial_{x} e\|_{L^2([0, T]\times \Omega)}, \|\partial_{x} g^{*}\|_{L^2(\Omega)}$, $\|\partial_{x}I\|_{L^2(\Omega)}$, and $\|f^{**}\|_{L^2(\Omega)}$ are bounded by a constant which depends only on $r, T$ and an uniform constant $\epsilon$ in the assumption. Therefore, by Gr\"onwall's inequality, we can conclude that there exists a constant $c(r, T,\Omega)$ such that 
     \begin{align*}
         \sup_{0\le t\le T}\|\partial_{x} e\|_{L^2(\Omega)}^{2}  \leq& \exp(\int_{0}^{T}\frac{C_1}{2\eta}\|\partial_{x}e\|_{L^2(\Omega)}+\frac{C_{1}}{4\eta}\|I\|_{H^1(\Omega)}dt)\cdot\\&(\|\partial_{x}g^{*}\|_{L^2(\Omega)}+\int_{0}^{T}\frac{1}{4\eta}\|f^{**}\|_{L^2(\Omega)}^{2}dt\le C(r, T, \epsilon)
     \end{align*}
     Now, if we integrate the inequality (\ref{Burgers_Gronwall_H2}) over $[0, T]$, the term $\|\partial_{x}^{2}e\|_{L^2([0, T]\times \Omega)}$ should be bounded since the terms $\|\partial_{x}e(0,\cdot)\|_{L^2(\Omega)}^{2}$ and $\|\partial_{x}e(T, \cdot)\|_{L^2(\Omega)}^{2}$ are bounded by $\sup_{0\le t\le T}\|\partial_{x}e\|_{L^2(\Omega)}^{2} $. Finally, the triangular inequality for \\$\|\partial_{x}^{2} w\|_{L^2([0, T]\times\Omega)}$ yields the boundedness of $\|\partial_{x}^{2}w\|_{L^2([0,T]\times\Omega)}$. On the one hand, 
     \begin{align*}
         \partial_{t} w= -w \partial_{x}w+\nu\partial_{x}^{2}w+f
     \end{align*}
    where 
    \begin{align*}
         \|w\partial_{x}w\|_{L^2([0, T]\times \Omega)}&\le\|\partial_{x}w\|_{L^2([0, T]\times \Omega)}\|w\|_{L^2([0, T]\times \Omega)}
         \\&\le T\sup_{0\le t \le T}\|\partial_{x}w\|_{L^2(\Omega)}\|w\|_{L^2([0, T]\times\Omega)} \\&\le T(\sup_{0\le t \le T }\|\partial_{x} I\|+\sup_{0\le t \le T}\|\partial_{x} e\|)\|w\|_{L^2([0, T]\times\Omega})^{2}.
    \end{align*}
    Therefore, $\partial_{t}w$ lies in ${L^2([0, T]\times \Omega})$ and we can conclude that $\|w\|_{L^2([0, T];H^2(\Omega))}$ and $\|\partial_{t} w\|_{L^2([0, T];L^2(\Omega)}$ are bounded by a constant which only depends on $r, T,\epsilon$,
\end{proof}

\textbf{Proof of Theorem \ref{equi} for Klein--Gordon equation}\\
The Klein--Gordon equation reads:
\begin{equation}\label{Klein_eq}
    \begin{aligned}
        &\partial_{t}^2 u - \alpha \partial_{x}^2 u+\beta u +\gamma u^k =f(t, x), &&(t, x)\in [0, T]\times\Omega\\
        &u(0, x)=g_{1}(x), && x \in \Omega\\
        &\partial_{t} u(0, x)=g_{2}(x), && x \in \Omega\\
        &u(t, a)=h_{1}(t), u(t, b)=h_{2}(t) ,&&  t\in [0, T]\\
    \end{aligned}
\end{equation}
where $\alpha, \beta, \gamma$ are positive constants with an odd number $k>0$. We consider the case where $\Omega$ is a one dimensional interval $[0, 1]$ and $k=3$. We assume the uniform boundedness of $\partial_{x} g_{1}^{*}, \partial_t h_1^{*}, \partial_t ^2 h_1^{*}$ and $\partial_t h_2^{*}, \partial_t ^2 h_2^{*}$ where $h_1^*,h_2^*$ and $g_1^*$ are boundary and initial conditions of the neural network solution $w$. The details are as follows.
\begin{assumption}
    There exists a constant $\epsilon>0$ such that $\int_{[0, T]} (\partial_{x} g_1)^{2}+(\partial_t h_1)^2 +(\partial_t ^2 h_1)^2+(\partial_t h_2)^2+(\partial_t ^2 h_2)^2 dt \le \epsilon$.
\end{assumption}
\begin{proof}
    Suppose that the neural network solution $w$ satisfies the following. 
    \begin{equation*}
        \begin{aligned}
            &\partial_{t}^2 w-\alpha \partial_{x}^2 w+\beta w +\gamma w^3 =f^{*}(t, x), &&(t, x)\in [0, T]\times\Omega\\
            &w(0, x)=g_{1}^{*}(x), && x \in \Omega\\
            &\partial_{t} w(0, x)=g_{2}^{*}(x), && x \in \Omega\\
            &w(t, a)=h_{1}^{*}(t), w(t, b)=h_{2}^{*}(t) ,&&  t\in [0, T]\\
        \end{aligned}
    \end{equation*}
    We again consider the interpolation function $I(t, x)=h_1(t)(1-x) +h_2(t)x$. Note that the norms of $I(t,x), \partial_t I(t,x)$ and $\partial_t^2 I(t, x)$ are bounded by a constant multiple of $\sum_{0\leq \alpha\le 2}||\partial_{t}^{\alpha}h_1||_{L^2([0,T])}+||\partial_{t}^{\alpha}h_2||_{L^2([0,T])}$ in  $L^2([0, T]; H^2(\Omega))$. Furthermore, the integrability of the square of nonlinear term $I(t,x)^{k}$ can be found as follows by applying Morrey's inequality. 
    \begin{align*}
        \int_{[0,T]\times \Omega} I^{2k}(t, x) dtdx &\leq ||I||^{2k-2}_{L^{\infty}([0,T]\times \Omega)}||I||^{2}_{L^2([0,T]\times \Omega)}\\&\leq C_1 ||I||^{2k-2}_{H^{1}([0, T] \times\Omega)}||I||^{2}_{L^2([0,T]\times \Omega)}  
    \end{align*}
    where $C_1$ is a constant depending only on $T$. On the other hand, Morrey's inequality yields that there exists a constant $C$ such that 
    \begin{align*}
        h_{1}(0), h_2(0) &\leq ||h_{1}||_{L^{\infty}([0, T])}+||h_2||_{L^{\infty}([0, T])} \\&\leq C ||h_1||_{H^1([0,T])} +||h_2||_{H^1([0, T])}, \\
        \partial_{t} h_{1}(0), \partial_{t} h_2(0) &\leq ||\partial_{t}h_{1}||_{L^{\infty}([0, T])}+||\partial_{t} h_2||_{L^{\infty}([0, T])} \\&\leq C ||h_1||_{H^2([0,T])} +||h_2||_{H^2([0, T])}.
    \end{align*}
    Thus, we can conclude that $u(0, x)-I(0, x)$ should be included in $L^2(\Omega)$ with $u(0,0)-I(0,0)=u(0,1)-I(0,1)=0$. Let $e(t,x):=w(t, x)-I(t,x)$. Now we estimate the terms in the equation below. 
    \begin{equation}\label{klein_gordon_approxim}
        \begin{aligned}
            &\partial_{t}^2 e-\alpha \partial_{x}^2 e+\beta e +\gamma(3I^2e +3Ie^2+e^3)=f^{**}, &&(t, x)\in [0, T]\times\Omega,\\
            &e(0, x)=g_{1}^{**}(x), && x \in \Omega,\\
            &\partial_{t} e(0, x)=g_{2}^{**}(x), && x \in \Omega,\\
            &e(t, a)=0, e(t, b)=0 ,&&  t\in [0, T],\\
        \end{aligned}
    \end{equation}
    where 
    \begin{align*}
        &f^{**}(t,x) = f^{*}(t, x)- (\partial_{t}^{2} I+ \alpha\partial_{x}^{2}I -\beta I -\gamma I^{3}), \\
        &g_{1}^{**}(x) = g_1^{*}(x)- g_1^{*}(0)(1-x)-g_1^{*}(1)x, \;\text{and} \\ &g_{2}^{**}(x) = g_2^{*}(x) -g_2^{*}(0)(1-x)- g_2^{*}x.
    \end{align*}
    Note that the $L^2$ norms of $g_1^{**},g_2^{**} \in L^2(\Omega)$ and $f^{**}\in L^2([0,T]\times\Omega)$ are bounded by $\|g_1\|_{L^2(\Omega)}, \|g_{2}\|_{L^2(\Omega)}$, $\|f\|_{L^2(\Omega)}$ and $||\partial_{t}^{\alpha}h_{i}||_{L^2([0,T])}$ through the above discussion involving Morrey's inequality. Let us define the functional $E(u)$ as below. We aim to estimate the following for arbitrary small $\eta>0$. 
    \begin{align*}
        E(u)=\int_{\Omega} |\partial_{t} e| ^2 +\alpha|\partial_{x} e|^2+\beta|e|^2+\gamma|e|^4 dx.
    \end{align*}
    Then, the following holds due to the fact that $e$ is sufficiently differentiable. 
    \begin{align*}
        &\frac{1}{2}\frac{d}{dt} E(u) \\
        &=\int _{\Omega} \partial_{t}^{2}e \partial_{t} e + \alpha \partial_{x}e\partial_{x}\partial_{t}e+\beta e\partial_{t}e +\gamma e^3 \partial_{t} e dx 
        \\&=\int _{\Omega} \partial_{t}^{2}e \partial_{t} e - \alpha \partial_{x}^{2}e\partial_{t}e+\beta e\partial_{t}e +\gamma e^3 \partial_{t} e dx  -\int_{\partial \Omega} \partial_{x}e\partial_{t}e \cdot n dS
        \\&= \int_{\Omega} (f^{**}-\gamma(3I^2e+3Ie^2))\partial_{t} e dx .
    \end{align*}
    With Young's inequality and H\"older's inequality, we can bound the three terms on the right hand side.
    \begin{align*}
         \int_{\Omega} f^{**}\partial_{t} e dx &\leq \frac{1}{2} |f^{**}|_{L^2(\Omega)}^{2}+\frac{1}{2} \int_{\Omega}|\partial_{t}e|^{2}dx,
        \\ \int_{\Omega} I^2 e \partial_{t} e dx &\leq \int_{\Omega}  I^4 e^2 dx +\frac{1}{4}\int_{\Omega} |\partial_{t}e|^{2} dx 
        \\ &\leq  \int_{\Omega}e^{4}dx+\frac{1}{4}(\|I^{4}\|^{2}_{L^2(\Omega)}+\int_{\Omega}|\partial_{t}e|^{2} dx),
        \\ \int_{\Omega} Ie^2 \partial_{t} e dx &\leq \int_{\Omega} I^2 e^{4} dx + \frac{1}{4}\int_{\Omega}|\partial_{t}e|^{2} dx \\&\leq \|I\|_{L^{\infty}(\Omega)}^{2}\int_{\Omega}e^{4} dx + \frac{1}{4}\int_{\Omega}|\partial_{t}e|^{2} dx.
    \end{align*}
    Consequently, we have the following inequality.
    \begin{align*}
        \frac{d}{dt}E(u) \leq (1+\frac{3}{2}\gamma+\|I\|_{L^{\infty}(\Omega)}^{2})E(u)+\frac{3}{4}\gamma\|I^{4}\|_{L^2(\Omega)}^{2}+\|f^{**}\|_{L^2(\Omega)}
    \end{align*}
    Finally, Gr\"onwall's inequality implies that 
    \begin{align*}
        \sup_{0\le t \le T}E(u)\leq & (\int_{0}^{T} 1+\frac{3}{2}\gamma+\|I\|^{2}_{L^{\infty}(\Omega)} dt) (\|g_{2}^{**}\|_{L^2(\Omega)}^{2}+\alpha\|\partial_{x}g_{1}^{**}\|_{L^2(\Omega)}^{2}\\&+\beta\|g_{1}^{**}\|_{L^{2}(\Omega)}^{2}+\gamma\|g_{1}^{**}\|_{L^4(\Omega)}^{4}+\int_{0}^{T} \|f^{**}\|_{L^2(\Omega)} )
    \end{align*}
    Note that $\int_{0}^{T} \|I\|_{L^{\infty}(\Omega)}^{2}$ can be bounded by a constant multiple of $\|h_{1}\|^{2}_{H^{2}([0, T])}+\|h_{2}\|^{2}_{H^2([0, T])}$ by the above discussion with Morrey's inequality. The Poincar\'e inequality yields an upper bound of $\|g_{1}^{**}\|_{L^4(\Omega)}$ which is a constant multiple of $\|g_{1}^{**}\|_{H^1(\Omega)}^{2}$.
    Now we suppose that the neural network $w$ such that $L_n(w)\leq r$ for some $n\in\mathbb{N}$. That is,
    \begin{align*}
        L_{n}(w) = &L(w) + \langle \lambda_1(x), w(0, x)-g_1(x)\rangle_{L^2(\Omega)} +\\&\langle \lambda_2(t), w_{t}(0, x)-g_2(x) \rangle_{L^2(\Omega)}+\\&\langle \lambda_3(t), w(t,a)-h_1(a)   \rangle_{L^2([0, T])} +\\& \langle \lambda_4(t), w(t,b)-h_2(b) \rangle_{L^2([0, T])}\le r,
    \end{align*}
    where
    \begin{align*}
        L(w) =& \|f-f^*\|_{L^2([0,T]\times\Omega)}^2 + \|g_1-g_1^{*}\|_{L^2(\Omega)}^2 + \|g_2-g_2^{*}\|_{L^2(\Omega)}^2+\\&\|h_1-h_1^{*}\|_{L^2([0,T])}^2 + \|h_2-h_2^{*}\|_{L^2([0,T])}^2.
    \end{align*}
    for some $r > 0$. If we develop similar arguments as before, we can conclude that the set $\{w\in L^2(\Omega)| L_{n}(w) \leq r $ for some $ n\in\mathbb{N}\}$ should be bounded in $H^{1,1}([0, T]\times \Omega)$.    
\end{proof}

\Bound*

\begin{proof}
    For an arbitrary small $\varepsilon>0$, there exists a neural network with width $n\in \mathbb{N}$ such that $\|Nu -f\|_{L^2(\Omega)}<\varepsilon$ by Theorem \ref{universal}. Since trace operator is continuous, $\|Tu_n -Tg\|_{L^2(\partial \Omega)}=\|Tu_{n}-Tu\|_{L^2(\partial \Omega)}$ is also bounded by a constant multiple of $\varepsilon$. Then, by the definition of quasi-minimizer, 
    \begin{align*}
        L_n (u_{n}) &\le \beta||Tu_{n}-g||_{L^2(\partial \Omega)}^2 +\langle \lambda_n, Tu_{n}-g \rangle_{L^2(\partial\Omega)}
        \\&\le \delta_{n}+O(\varepsilon).   
    \end{align*}
    We have, 
    \begin{align*}
        \|\lambda_{n+1}\|_{L^2(\partial\Omega)} &= \langle \lambda_{n+1}, \lambda_{n+1}\rangle_{L^2(\partial\Omega)} \\
        =&\langle \lambda_n + \eta_{\lambda}(Tu_n -g), \lambda_n + \eta_{\lambda}(Tu_n -g) \rangle_{L^2(\partial \Omega)} \\
        =&\langle \lambda_n, \lambda_n \rangle_{L^2(\partial \Omega)}+2\eta_{\lambda}\langle \lambda_n, Tu_n- g \rangle_{L^2(\partial \Omega)}+\\& \eta_{\lambda}^2\langle Tu_n -g, Tu_n -g \rangle_{L^2(\partial \Omega)}
        \\\leq & \| \lambda_n \|_{L^2(\partial \Omega)}
        +2\eta_{\lambda}\delta_n +(\eta_{\lambda}^2-2\eta_{\lambda}\beta)||Tu_n - g||_{L^2(\partial \Omega)}^2
    \end{align*}
    Therefore, if $\beta$ is sufficiently large so that $\beta > \frac{1}{2} (\eta_{\lambda} + \frac{2\delta_n}{\|Tu_n-g\|_{L^2(\partial\Omega)}})$ for all $n\in\mathbb{N}$, then $\{\|\lambda_n\|_{L^2(\partial \Omega)}\}_{n\in\mathbb{N}}$ is a decreasing sequence. In conclusion, $\left\{\lambda_n\right\}$ must be a bounded sequence in $L^2(\partial \Omega)$.
\end{proof}

\textbf{Proof Sketch of equi-coercivity for linear elliptic PDE}
For $Nu:=-\sum_{i, j=1}^{n} (a^{ij}(x)u_{x_i})_{x_j}+\sum_{i=1}^{n}b^{i}(x)u_{x_i}+c(x)u$, we consider a class of PDEs that reads : 
\begin{equation*}
    \begin{split}
        &Nu = f\in L^2(\Omega), \text{ for } x\in\Omega, \\
        &Tu = g\in L^2(\partial \Omega), \text{ for } x\in\partial\Omega.
    \end{split}
\end{equation*} 
with the assumption that there exist positive constant $\theta$ and $\eta>0$ such that
\begin{align*}
    \sum_{i, j =1} ^{n} a^{ij}(x) \xi_{i} \xi_{j} \geq \theta \|\xi\|^2,\quad \text{and} \quad c(x)\geq \eta>0, 
\end{align*}
$\forall x\in\Omega, \xi=(\xi_1, \cdots, \xi_n)\in \mathbb{R}^n$. In the AL-PINN method for elliptic partial differential equations, it is not straightforward to show that the sequence of loss functional is equi-coercive. We leave here a brief sketch of proof of the property. 

The classical Poincare inequality considers the zero trace functions which satisfy the boundary condition. If the uniform boundedness of $||\lambda_n||_{L^{2}(\partial \Omega)}$ is carefully considered, the inequality can also be extended to an augmented Lagrangian functional through the standard argument. The detailed statements and proofs are as follows.

\begin{lemma}
    For a positive real constant $r$, consider the set $M_n =\left\{ u\in H^1 (\Omega)| L_n (u) \le r \right\}$. Then, there exists a constant $C$, which depends only on $r$ and $\Omega$ (not on $n$), such that 
    \begin{align*}
        ||u||_{L^2(\Omega)}\le C\cdot (||\nabla u||_{L^2(\Omega)}+1) \hspace{1cm} \forall u\in M_n . 
    \end{align*}
\end{lemma}

\begin{proof}
    Suppose that there is no constant that satisfies the inequality. That is, there exists a sequence $u_k \in M_n$ such that 
    \begin{align*}
        ||u_k||_{L^2(\Omega)}> k (||\nabla u_k||_{L^2{(\Omega)}}+1).
    \end{align*}
    For an arbitrary constant $R>0$, a sequence $\delta_k = R/||u_k||_{L^2(\partial \Omega)}$ converges to 0 since $||u_k||_{L^2(\partial \Omega)} \rightarrow \infty$. Therefore, $\delta_k \le 1$ for a sufficiently large $k$. The set $M_n$ contains the zero-function for a positive real constant $r$ and has convexity so that a sequence $v_k = \delta_k \cdot u_k$ lies in $M_n$ for a sufficiently large $k$. Furthermore, $v_k$ is bounded in $H^1(\Omega)$ since $||v_k||_{L^2 (\partial \Omega)}=R$, and
    \begin{align*}
        ||\nabla v_k||_{L^2 (\Omega)}&=\delta_k||\nabla u_k||_{L^2 (\Omega)}\\&<\delta_k \left(\frac{||u_k||_{L^2(\Omega)}}{k} - 1 \right) = \frac{R}{k}  -\delta_k \rightarrow 0,
    \end{align*}
    as $k\rightarrow \infty$. The Sobolev space $H^1$ is reflexive so that a bounded sequence contains a weakly convergent subsequence. For a weakly convergent subsequence $\left\{v_{k_{l}}\right\}_{l\in \mathbb{N}}$, the limit $v$ lies in the set $M_n$ which is a closed, convex and therefore, weakly sequentially closed set. Also, the strong convergence of $\nabla v_k$ to 0 and the uniqueness of limit imply that $\nabla v=0$ in $L^2(\Omega)$. Through the standard arguments of mollifiers, $v$ must be a constant $C$ a.e. when $\Omega$ is connected. \\Now, we apply the Rellich-Kondrachov compactness theorem for $1/2>1/2-1/n$ which states the embedding from $H^1(\Omega)$ to $L^2(\Omega)$ is a compact operator, and therefore it maps weakly convergent sequence to strongly convergent sequence. Consequently, $v_k$ converges to $v$ strongly in $L^2(\Omega)$, and by the continuity of $L^2(\Omega)$ norm on its space,
    \begin{align*}
        R=||v_k||_{L^2(\Omega)}=\lim_{k\rightarrow\infty} ||v_k||_{L^2(\Omega)}=||v||_{L^2(\Omega)} = C|\Omega|. 
    \end{align*}
    In conclusion, the set $M_n$ contains all arbitrarily constant functions since there is no constraint for choosing $R$. However, it cannot be achieved since it means that for an arbitrary constant function $w$ ($= c$ a.e.),  
    \begin{align*}
        r &\ge L_n (w) \\&\ge||Nw-f||_{L^2(\Omega)}^2 + \beta ||Tw-g||_{L^2{(\partial\Omega})}^2 + \langle \lambda_n ,  Tw-g\rangle_{L^2{(\partial \Omega)}}
        \\&\ge \beta||Tw-g||^2_{L^2(\partial \Omega)}-||\lambda_n||_{L^2(\partial \Omega)}||Tw-g||_{L^2(\partial \Omega)}
        \\&\ge \beta(|c| - ||g||_{L^2(\partial \Omega)})^2 -||\lambda_n||_{L^2(\partial \Omega)}(|c|+||g||_{L^2(\partial \Omega)}) 
    \end{align*}
    where the last equation, which is a quadratic equation for $c$, blows up when $c$ goes to infinity.
\end{proof}

Recall the integration by part $\int_{\Omega}u_{x_{i}x_{i}}u dx=-\int_{\Omega}u_{x_{i}}u_{x_i}dx+\int_{\Omega}u_{x_{i}}u \nu ^{i} dS$ where $\nu^{i}$ is the $i$-th component of the outing unit normal vector field for $\partial \Omega$. Following the proof of Theorem 2 in Section 6 of \cite{evans1998partial} with the above generalized Poincare inequality, we can obtain the desired results for a class of linear elliptic equations.

\textbf{Trace Theorem for the viscous Burgers and Klein--Gordon equations}
\\To generalize the discussions in Theorem \ref{conv_PDE} and Lemma \ref{Bound} to the viscous Burgers and Klein--Gordon equations, we need the continuity of the trace operators. For a complete explanation, we introduce some definitions and theorems. For a function $u(t, x):[0, T]\times \Omega\rightarrow \mathbb{R}$ with $u(t_1, \cdot)\in H^2(\Omega)$ for every $t_{1}\in[0, T]$, $\mathbf{u}:[0, T]\rightarrow H^2(\Omega)$ is defined by $\mathbf{u}(t):=u(t, x)$. Let $\mathbf{u}':[0, T]\rightarrow H(\Omega)$ be a function satisfying 
\begin{align*}
     \int_0^{T} \phi'(t)\mathbf{u}(t) dt = -\int_{0}^{T} \phi(t) \mathbf{u'}(t) dt
\end{align*}
for every $\phi(t)\in C_{c}^{\infty}([0, T])$. Then, the following holds.
\begin{theorem}[Thm 4 in Sec 5.9 of \cite{evans1998partial}] \label{Regularity_Time}
Suppose that $\Omega$ is an open and bounded set with smooth boundary $\partial \Omega$. For $\mathbf{u}\in L^2([0, T];H^2(\Omega))$ and $\mathbf{u}'\in L^2([0,T];L^2(\Omega))$, there exists a constant $C(T, \Omega)$ such that 
\begin{align*}
    \max_{0\le t \le T} \|\mathbf{u}(t)\|_{H^{1}(\Omega)}\le C(T, \Omega)(\| \mathbf{u}\|_{L^2([0, T]; H^2(\Omega))}+\|\mathbf{u}'\|_{L^2([0, T];L^2(\Omega))})
\end{align*}
\end{theorem} 
Motivated by the above theorem, we can extend the trace theorem to the equation with the Dirichlet boundary condition. Consider the case when a spatial domain is an interval $[0, 1]\in \mathbb{R}$. Then, the following holds.
\begin{corollary}
    For $\mathbf{u}\in L^2([0,T]; H^2(\Omega))$ and \\ $\mathbf{u}'\in L^2([0, T];L^2(\Omega))$, let us define boundary functions $h_{1}, h_{2}:[0, T]\rightarrow \mathbb{R}$ and a initial function $g:\Omega \rightarrow\mathbb{R}$ by 
    \begin{align*}
        h_1(t):=trace(\mathbf{u}(t))(0), h_2(t):=trace(\mathbf{u}(t))(1) \hspace{0.3cm} \text{and} \hspace {0.3cm} g(x) := \mathbf{u}(0).
    \end{align*}
    where $trace:W^{1,2}(\Omega)\rightarrow L^2(\partial \Omega)$ denotes the trace operator. Then, there exists a constant $C$ depending only on $T, \Omega$ such that the following inequality holds.
    \begin{align*}
        \|h_{1}(t)\|_{L^2([0, T])}+\|h_{2}(t)\|_{L^2([0, T])}+\|g(x)\|_{L^2(\Omega)} \\ \leq  C(\| \mathbf{u}\|_{L^2([0, T]; H^2(\Omega))}+ \|\mathbf{u}'\|_{L^2([0, T];L^2(\Omega))}).
    \end{align*}
\end{corollary}
\begin{proof}
    Let us denote the constant in theorem \ref{Regularity_Time} by $C(T, \Omega)$. Then, we can estimate $\|g(x)\|_{L^2(\Omega)}$ as follows. 
    \begin{align*}
        \|g(x)\|_{L^2(\Omega)}\leq \|g(x)\|_{H^1(\Omega)}\leq  \max_{0\le t \le T}\|\mathbf{u}(t)\|_{H^1(\Omega)}.
    \end{align*}
    On the one hand, we can estimate other two terms applying the usual trace theorem.
    \begin{align*}
        \|h_{1}(t)\|_{L^2([0, T])}+\|h_2(t)\|_{L^2([0, T])}&\leq T\max_{0\le t\le T}(h_{1}(t)+h_2(t)) \\& \leq \sqrt{2}T C_{trace}\max_{0\le t\le T}(\|\mathbf{u}\|_{H^1(\Omega)}),
    \end{align*}
    where $C_{trace}$ denotes a constant in the trace theorem which depends only on $\Omega$, not on $t$. Combining two inequalities, we get the desired property.
\end{proof}

For the Klein--Gordon equation, we introduce the following trace theorem on a Lipschitz domain.
\begin{theorem}[Thm 1.I in \cite{gagliardo1957caratterizzazioni}]
Suppose that $\Omega\in \mathbb{R}^{n}$ is a bounded set with Lipschitz boundary $\partial \Omega$.. Then, there exists a bounded linear operator $T:W^{1,p}(\Omega)\rightarrow L^{p}(\partial\Omega)$ and a constant $C$ such that $\|T(u)\|_{L^{p}(\partial\Omega)}\le C\|u\|_{W^{1,p}(\Omega)}$ and,
\begin{align*}
    T(u)=u|_{\partial\Omega} \hspace{0.3cm} \text{for} \hspace{0.1cm} u\in W^{1,p}(\Omega)\cap  L^{p}(\Omega),
\end{align*}
where $C$ depends only on $p$ and $\Omega$. 
\end{theorem}

Recall that we constructed the sequence $L_{n}(w)$ of loss functional as below for the viscous Burgers' equation.
\begin{align*}
    L_{n}(w) =& L(w) + \langle \lambda_1(x), w(0, x)-u_0(x)\rangle_{L^2(\Omega)} +\\& \langle \lambda_2(t), w(t,a)   \rangle_{L^2([0, T])}  +\langle \lambda_3(t), w(t,b) \rangle_{L^2([0, T])},
    \end{align*}
    where
    \begin{align*}
        L(w) =& \|w_{t}+ww_{x}-\nu w_{xx}\|_{L^2([0,T]\times\Omega)}^2 +\beta\|w(0, x)-u_0(x)\|_{L^2(\Omega)}^2 +\\&\beta \|w(t,a)\|_{L^2([0,T])}^2 +\beta \|w(t,b)\|_{L^2([0,T])}^2.
    \end{align*}
    Define a function $\lambda:[0, T]\times\{0\}$ $\bigcup$ $\{0\}\times\Omega$ $\bigcup $ $[0,T]\times\{1\}$ by 
    \begin{align*}
        \lambda(t, 0)=\lambda_{2}(t), \lambda(t, 1)=\lambda_{3}(t)\hspace{0.3cm} \text{and}, \hspace{0.3cm} \lambda(0, x)=\lambda_{1}(x),
    \end{align*} $\forall t\in [0, T], x\in \Omega.$. Then, the update rules for $\lambda_{1}, \lambda_{2}$, and $\lambda_{3}$ during the training can be transformed into an update rule for $\lambda$. In this setting, we can apply all of the previous arguments to prove Theorem \ref{conv_PDE} and Lemma \ref{Bound} for the viscous Burgers' equation.


 \bibliographystyle{elsarticle-num-names} 
 \bibliography{example_paper}

\begin{thebibliography}{49}
\expandafter\ifx\csname natexlab\endcsname\relax\def\natexlab#1{#1}\fi
\providecommand{\url}[1]{\texttt{#1}}
\providecommand{\href}[2]{#2}
\providecommand{\path}[1]{#1}
\providecommand{\DOIprefix}{doi:}
\providecommand{\ArXivprefix}{arXiv:}
\providecommand{\URLprefix}{URL: }
\providecommand{\Pubmedprefix}{pmid:}
\providecommand{\doi}[1]{\href{http://dx.doi.org/#1}{\path{#1}}}
\providecommand{\Pubmed}[1]{\href{pmid:#1}{\path{#1}}}
\providecommand{\bibinfo}[2]{#2}
\ifx\xfnm\relax \def\xfnm[#1]{\unskip,\space#1}\fi
\bibitem[{Raissi et~al.(2019)Raissi, Perdikaris, and
  Karniadakis}]{raissi2019physics}
\bibinfo{author}{M.~Raissi}, \bibinfo{author}{P.~Perdikaris},
  \bibinfo{author}{G.~E. Karniadakis},
\newblock \bibinfo{title}{Physics-informed neural networks: A deep learning
  framework for solving forward and inverse problems involving nonlinear
  partial differential equations},
\newblock \bibinfo{journal}{Journal of Computational Physics}
  \bibinfo{volume}{378} (\bibinfo{year}{2019}) \bibinfo{pages}{686--707}.
\bibitem[{Lu et~al.(2021)Lu, Meng, Mao, and Karniadakis}]{lu2021deepxde}
\bibinfo{author}{L.~Lu}, \bibinfo{author}{X.~Meng}, \bibinfo{author}{Z.~Mao},
  \bibinfo{author}{G.~E. Karniadakis},
\newblock \bibinfo{title}{Deepxde: A deep learning library for solving
  differential equations},
\newblock \bibinfo{journal}{SIAM Review} \bibinfo{volume}{63}
  (\bibinfo{year}{2021}) \bibinfo{pages}{208--228}.
\bibitem[{Karniadakis et~al.(2021)Karniadakis, Kevrekidis, Lu, Perdikaris,
  Wang, and Yang}]{karniadakis2021physics}
\bibinfo{author}{G.~E. Karniadakis}, \bibinfo{author}{I.~G. Kevrekidis},
  \bibinfo{author}{L.~Lu}, \bibinfo{author}{P.~Perdikaris},
  \bibinfo{author}{S.~Wang}, \bibinfo{author}{L.~Yang},
\newblock \bibinfo{title}{Physics-informed machine learning},
\newblock \bibinfo{journal}{Nature Reviews Physics} \bibinfo{volume}{3}
  (\bibinfo{year}{2021}) \bibinfo{pages}{422--440}.
\bibitem[{Shin et~al.(2020)Shin, Darbon, and
  Em~Karniadakis}]{shin2020convergence}
\bibinfo{author}{Y.~Shin}, \bibinfo{author}{J.~Darbon},
  \bibinfo{author}{G.~Em~Karniadakis},
\newblock \bibinfo{title}{On the convergence of physics informed neural
  networks for linear second-order elliptic and parabolic type pdes},
\newblock \bibinfo{journal}{Communications in Computational Physics}
  \bibinfo{volume}{28} (\bibinfo{year}{2020}) \bibinfo{pages}{2042--2074}.
  \URLprefix \url{http://global-sci.org/intro/article_detail/cicp/18404.html}.
  \DOIprefix\doi{https://doi.org/10.4208/cicp.OA-2020-0193}.
\bibitem[{Jagtap et~al.(2022)Jagtap, Shin, Kawaguchi, and
  Karniadakis}]{jagtap2022deep}
\bibinfo{author}{A.~D. Jagtap}, \bibinfo{author}{Y.~Shin},
  \bibinfo{author}{K.~Kawaguchi}, \bibinfo{author}{G.~E. Karniadakis},
\newblock \bibinfo{title}{Deep kronecker neural networks: A general framework
  for neural networks with adaptive activation functions},
\newblock \bibinfo{journal}{Neurocomputing} \bibinfo{volume}{468}
  (\bibinfo{year}{2022}) \bibinfo{pages}{165--180}.
\bibitem[{Sirignano and Spiliopoulos(2018)}]{sirignano2018dgm}
\bibinfo{author}{J.~Sirignano}, \bibinfo{author}{K.~Spiliopoulos},
\newblock \bibinfo{title}{Dgm: A deep learning algorithm for solving partial
  differential equations},
\newblock \bibinfo{journal}{Journal of computational physics}
  \bibinfo{volume}{375} (\bibinfo{year}{2018}) \bibinfo{pages}{1339--1364}.
\bibitem[{Jo et~al.(2020)Jo, Son, Hwang, and Kim}]{jo2020deep}
\bibinfo{author}{H.~Jo}, \bibinfo{author}{H.~Son}, \bibinfo{author}{H.~J.
  Hwang}, \bibinfo{author}{E.~H. Kim},
\newblock \bibinfo{title}{Deep neural network approach to forward-inverse
  problems},
\newblock \bibinfo{journal}{Networks \& Heterogeneous Media}
  \bibinfo{volume}{15} (\bibinfo{year}{2020}) \bibinfo{pages}{247}.
\bibitem[{Hwang et~al.(2020)Hwang, Jang, Jo, and Lee}]{hwang2020trend}
\bibinfo{author}{H.~J. Hwang}, \bibinfo{author}{J.~W. Jang},
  \bibinfo{author}{H.~Jo}, \bibinfo{author}{J.~Y. Lee},
\newblock \bibinfo{title}{Trend to equilibrium for the kinetic fokker-planck
  equation via the neural network approach},
\newblock \bibinfo{journal}{Journal of Computational Physics}
  \bibinfo{volume}{419} (\bibinfo{year}{2020}) \bibinfo{pages}{109665}.
\bibitem[{Yu et~al.(2017)}]{yu2017deep}
\bibinfo{author}{B.~Yu}, et~al.,
\newblock \bibinfo{title}{The deep ritz method: a deep learning-based numerical
  algorithm for solving variational problems},
\newblock \bibinfo{journal}{arXiv preprint arXiv:1710.00211}
  (\bibinfo{year}{2017}).
\bibitem[{Chen et~al.(2020)Chen, Lu, Karniadakis, and
  Dal~Negro}]{chen2020physics}
\bibinfo{author}{Y.~Chen}, \bibinfo{author}{L.~Lu}, \bibinfo{author}{G.~E.
  Karniadakis}, \bibinfo{author}{L.~Dal~Negro},
\newblock \bibinfo{title}{Physics-informed neural networks for inverse problems
  in nano-optics and metamaterials},
\newblock \bibinfo{journal}{Optics express} \bibinfo{volume}{28}
  (\bibinfo{year}{2020}) \bibinfo{pages}{11618--11633}.
\bibitem[{M{\"u}ller and Zeinhofer(2019)}]{muller2019deep}
\bibinfo{author}{J.~M{\"u}ller}, \bibinfo{author}{M.~Zeinhofer},
\newblock \bibinfo{title}{Deep ritz revisited},
\newblock \bibinfo{journal}{arXiv preprint arXiv:1912.03937}
  (\bibinfo{year}{2019}).
\bibitem[{Huang et~al.(2021)Huang, Wang, and Zhou}]{huang2021augmented}
\bibinfo{author}{J.~Huang}, \bibinfo{author}{H.~Wang},
  \bibinfo{author}{T.~Zhou},
\newblock \bibinfo{title}{An augmented lagrangian deep learning method for
  variational problems with essential boundary conditions},
\newblock \bibinfo{journal}{arXiv preprint arXiv:2106.14348}
  (\bibinfo{year}{2021}).
\bibitem[{M{\'a}rquez-Neila et~al.(2017)M{\'a}rquez-Neila, Salzmann, and
  Fua}]{marquez2017imposing}
\bibinfo{author}{P.~M{\'a}rquez-Neila}, \bibinfo{author}{M.~Salzmann},
  \bibinfo{author}{P.~Fua},
\newblock \bibinfo{title}{Imposing hard constraints on deep networks: Promises
  and limitations},
\newblock \bibinfo{journal}{arXiv preprint arXiv:1706.02025}
  (\bibinfo{year}{2017}).
\bibitem[{Nandwani et~al.(2019)Nandwani, Pathak, Singla
  et~al.}]{nandwani2019primal}
\bibinfo{author}{Y.~Nandwani}, \bibinfo{author}{A.~Pathak},
  \bibinfo{author}{P.~Singla}, et~al.,
\newblock \bibinfo{title}{A primal dual formulation for deep learning with
  constraints},
\newblock in: \bibinfo{booktitle}{Advances in Neural Information Processing
  Systems}, \bibinfo{year}{2019}, pp. \bibinfo{pages}{12157--12168}.
\bibitem[{Sangalli et~al.(2021)Sangalli, Erdil, Hoetker, Donati, and
  Konukoglu}]{sangalli2021constrained}
\bibinfo{author}{S.~Sangalli}, \bibinfo{author}{E.~Erdil},
  \bibinfo{author}{A.~Hoetker}, \bibinfo{author}{O.~Donati},
  \bibinfo{author}{E.~Konukoglu},
\newblock \bibinfo{title}{Constrained optimization for training deep neural
  networks under class imbalance},
\newblock \bibinfo{journal}{arXiv preprint arXiv:2102.12894}
  (\bibinfo{year}{2021}).
\bibitem[{Fioretto et~al.(2020)Fioretto, Van~Hentenryck, Mak, Tran, Baldo, and
  Lombardi}]{fioretto2020lagrangian}
\bibinfo{author}{F.~Fioretto}, \bibinfo{author}{P.~Van~Hentenryck},
  \bibinfo{author}{T.~W. Mak}, \bibinfo{author}{C.~Tran},
  \bibinfo{author}{F.~Baldo}, \bibinfo{author}{M.~Lombardi},
\newblock \bibinfo{title}{Lagrangian duality for constrained deep learning},
\newblock \bibinfo{journal}{arXiv preprint arXiv:2001.09394}
  (\bibinfo{year}{2020}).
\bibitem[{Hwang and Son(2021)}]{hwang2021lagrangian}
\bibinfo{author}{H.~J. Hwang}, \bibinfo{author}{H.~Son},
\newblock \bibinfo{title}{Lagrangian dual framework for conservative neural
  network solutions of kinetic equations},
\newblock \bibinfo{journal}{arXiv preprint arXiv:2106.12147}
  (\bibinfo{year}{2021}).
\bibitem[{Lu et~al.(2021)Lu, Pestourie, Yao, Wang, Verdugo, and
  Johnson}]{lu2021physics}
\bibinfo{author}{L.~Lu}, \bibinfo{author}{R.~Pestourie},
  \bibinfo{author}{W.~Yao}, \bibinfo{author}{Z.~Wang},
  \bibinfo{author}{F.~Verdugo}, \bibinfo{author}{S.~G. Johnson},
\newblock \bibinfo{title}{Physics-informed neural networks with hard
  constraints for inverse design},
\newblock \bibinfo{journal}{arXiv preprint arXiv:2102.04626}
  (\bibinfo{year}{2021}).
\bibitem[{Kim et~al.(2021)Kim, Lee, Lee, Jhin, and Park}]{kim2021dpm}
\bibinfo{author}{J.~Kim}, \bibinfo{author}{K.~Lee}, \bibinfo{author}{D.~Lee},
  \bibinfo{author}{S.~Y. Jhin}, \bibinfo{author}{N.~Park},
\newblock \bibinfo{title}{Dpm: A novel training method for physics-informed
  neural networks in extrapolation},
\newblock in: \bibinfo{booktitle}{Proceedings of the AAAI Conference on
  Artificial Intelligence}, volume~\bibinfo{volume}{35}, \bibinfo{year}{2021},
  pp. \bibinfo{pages}{8146--8154}.
\bibitem[{Lagaris et~al.(1998)Lagaris, Likas, and
  Fotiadis}]{lagaris1998artificial}
\bibinfo{author}{I.~E. Lagaris}, \bibinfo{author}{A.~Likas},
  \bibinfo{author}{D.~I. Fotiadis},
\newblock \bibinfo{title}{Artificial neural networks for solving ordinary and
  partial differential equations},
\newblock \bibinfo{journal}{IEEE transactions on neural networks}
  \bibinfo{volume}{9} (\bibinfo{year}{1998}) \bibinfo{pages}{987--1000}.
\bibitem[{Berg and Nystr{\"o}m(2018)}]{berg2018unified}
\bibinfo{author}{J.~Berg}, \bibinfo{author}{K.~Nystr{\"o}m},
\newblock \bibinfo{title}{A unified deep artificial neural network approach to
  partial differential equations in complex geometries},
\newblock \bibinfo{journal}{Neurocomputing} \bibinfo{volume}{317}
  (\bibinfo{year}{2018}) \bibinfo{pages}{28--41}.
\bibitem[{Son et~al.(2021)Son, Jang, Han, and Hwang}]{son2021sobolev}
\bibinfo{author}{H.~Son}, \bibinfo{author}{J.~W. Jang}, \bibinfo{author}{W.~J.
  Han}, \bibinfo{author}{H.~J. Hwang},
\newblock \bibinfo{title}{Sobolev training for the neural network solutions of
  pdes},
\newblock \bibinfo{journal}{arXiv preprint arXiv:2101.08932}
  (\bibinfo{year}{2021}).
\bibitem[{Sukumar and Srivastava(2021)}]{sukumar2021exact}
\bibinfo{author}{N.~Sukumar}, \bibinfo{author}{A.~Srivastava},
\newblock \bibinfo{title}{Exact imposition of boundary conditions with distance
  functions in physics-informed deep neural networks},
\newblock \bibinfo{journal}{arXiv preprint arXiv:2104.08426}
  (\bibinfo{year}{2021}).
\bibitem[{Schiassi et~al.(2021)Schiassi, Furfaro, Leake, De~Florio, Johnston,
  and Mortari}]{schiassi2021extreme}
\bibinfo{author}{E.~Schiassi}, \bibinfo{author}{R.~Furfaro},
  \bibinfo{author}{C.~Leake}, \bibinfo{author}{M.~De~Florio},
  \bibinfo{author}{H.~Johnston}, \bibinfo{author}{D.~Mortari},
\newblock \bibinfo{title}{Extreme theory of functional connections: A fast
  physics-informed neural network method for solving ordinary and partial
  differential equations},
\newblock \bibinfo{journal}{Neurocomputing} \bibinfo{volume}{457}
  (\bibinfo{year}{2021}) \bibinfo{pages}{334--356}.
\bibitem[{Zhao(2020)}]{zhao2020solving}
\bibinfo{author}{C.~L. Zhao},
\newblock \bibinfo{title}{Solving allen-cahn and cahn-hilliard equations using
  the adaptive physics informed neural networks},
\newblock \bibinfo{journal}{Communications in Computational Physics}
  \bibinfo{volume}{29} (\bibinfo{year}{2020}).
\bibitem[{McClenny and Braga-Neto(2020)}]{mcclenny2020self}
\bibinfo{author}{L.~McClenny}, \bibinfo{author}{U.~Braga-Neto},
\newblock \bibinfo{title}{Self-adaptive physics-informed neural networks using
  a soft attention mechanism},
\newblock \bibinfo{journal}{arXiv preprint arXiv:2009.04544}
  (\bibinfo{year}{2020}).
\bibitem[{Xiang et~al.(2022)Xiang, Peng, Liu, and Yao}]{xiang2022self}
\bibinfo{author}{Z.~Xiang}, \bibinfo{author}{W.~Peng},
  \bibinfo{author}{X.~Liu}, \bibinfo{author}{W.~Yao},
\newblock \bibinfo{title}{Self-adaptive loss balanced physics-informed neural
  networks},
\newblock \bibinfo{journal}{Neurocomputing} \bibinfo{volume}{496}
  (\bibinfo{year}{2022}) \bibinfo{pages}{11--34}.
\bibitem[{Wang et~al.(2021)Wang, Teng, and Perdikaris}]{wang2021understanding}
\bibinfo{author}{S.~Wang}, \bibinfo{author}{Y.~Teng},
  \bibinfo{author}{P.~Perdikaris},
\newblock \bibinfo{title}{Understanding and mitigating gradient flow
  pathologies in physics-informed neural networks},
\newblock \bibinfo{journal}{SIAM Journal on Scientific Computing}
  \bibinfo{volume}{43} (\bibinfo{year}{2021}) \bibinfo{pages}{A3055--A3081}.
\bibitem[{Wang et~al.(2022)Wang, Yu, and Perdikaris}]{wang2022and}
\bibinfo{author}{S.~Wang}, \bibinfo{author}{X.~Yu},
  \bibinfo{author}{P.~Perdikaris},
\newblock \bibinfo{title}{When and why pinns fail to train: A neural tangent
  kernel perspective},
\newblock \bibinfo{journal}{Journal of Computational Physics}
  \bibinfo{volume}{449} (\bibinfo{year}{2022}) \bibinfo{pages}{110768}.
\bibitem[{van~der Meer et~al.(2020)van~der Meer, Oosterlee, and
  Borovykh}]{van2020optimally}
\bibinfo{author}{R.~van~der Meer}, \bibinfo{author}{C.~Oosterlee},
  \bibinfo{author}{A.~Borovykh},
\newblock \bibinfo{title}{Optimally weighted loss functions for solving pdes
  with neural networks},
\newblock \bibinfo{journal}{arXiv preprint arXiv:2002.06269}
  (\bibinfo{year}{2020}).
\bibitem[{Bischof and Kraus(2021)}]{bischof2021multi}
\bibinfo{author}{R.~Bischof}, \bibinfo{author}{M.~Kraus},
\newblock \bibinfo{title}{Multi-objective loss balancing for physics-informed
  deep learning},
\newblock \bibinfo{journal}{arXiv preprint arXiv:2110.09813}
  (\bibinfo{year}{2021}).
\bibitem[{Rohrhofer et~al.(2021)Rohrhofer, Posch, and
  Geiger}]{rohrhofer2021pareto}
\bibinfo{author}{F.~M. Rohrhofer}, \bibinfo{author}{S.~Posch},
  \bibinfo{author}{B.~C. Geiger},
\newblock \bibinfo{title}{On the pareto front of physics-informed neural
  networks},
\newblock \bibinfo{journal}{arXiv preprint arXiv:2105.00862}
  (\bibinfo{year}{2021}).
\bibitem[{Boyd et~al.(2004)Boyd, Boyd, and Vandenberghe}]{boyd2004convex}
\bibinfo{author}{S.~Boyd}, \bibinfo{author}{S.~P. Boyd},
  \bibinfo{author}{L.~Vandenberghe}, \bibinfo{title}{Convex optimization},
  \bibinfo{publisher}{Cambridge university press}, \bibinfo{year}{2004}.
\bibitem[{Bertsekas(1976)}]{bertsekas1976multiplier}
\bibinfo{author}{D.~P. Bertsekas},
\newblock \bibinfo{title}{Multiplier methods: A survey},
\newblock \bibinfo{journal}{Automatica} \bibinfo{volume}{12}
  (\bibinfo{year}{1976}) \bibinfo{pages}{133--145}.
\bibitem[{Basir and Senocak(2022)}]{basir2022physics}
\bibinfo{author}{S.~Basir}, \bibinfo{author}{I.~Senocak},
\newblock \bibinfo{title}{Physics and equality constrained artificial neural
  networks: application to forward and inverse problems with multi-fidelity
  data fusion},
\newblock \bibinfo{journal}{Journal of Computational Physics}
  \bibinfo{volume}{463} (\bibinfo{year}{2022}) \bibinfo{pages}{111301}.
\bibitem[{M{\"u}ller and Zeinhofer(2021)}]{muller2021notes}
\bibinfo{author}{J.~M{\"u}ller}, \bibinfo{author}{M.~Zeinhofer},
\newblock \bibinfo{title}{Notes on exact boundary values in residual
  minimisation},
\newblock \bibinfo{journal}{arXiv preprint arXiv:2105.02550}
  (\bibinfo{year}{2021}).
\bibitem[{Dal~Maso(2012)}]{dal2012introduction}
\bibinfo{author}{G.~Dal~Maso}, \bibinfo{title}{An introduction to
  $\Gamma$-convergence}, volume~\bibinfo{volume}{8},
  \bibinfo{publisher}{Springer Science \& Business Media},
  \bibinfo{year}{2012}.
\bibitem[{Li(1996)}]{li1996simultaneous}
\bibinfo{author}{X.~Li},
\newblock \bibinfo{title}{Simultaneous approximations of multivariate functions
  and their derivatives by neural networks with one hidden layer},
\newblock \bibinfo{journal}{Neurocomputing} \bibinfo{volume}{12}
  (\bibinfo{year}{1996}) \bibinfo{pages}{327--343}.
\bibitem[{Kingma and Ba(2014)}]{kingma2014adam}
\bibinfo{author}{D.~P. Kingma}, \bibinfo{author}{J.~Ba},
\newblock \bibinfo{title}{Adam: A method for stochastic optimization},
\newblock \bibinfo{journal}{arXiv preprint arXiv:1412.6980}
  (\bibinfo{year}{2014}).
\bibitem[{He et~al.(2015)He, Zhang, Ren, and Sun}]{he2015delving}
\bibinfo{author}{K.~He}, \bibinfo{author}{X.~Zhang}, \bibinfo{author}{S.~Ren},
  \bibinfo{author}{J.~Sun},
\newblock \bibinfo{title}{Delving deep into rectifiers: Surpassing human-level
  performance on imagenet classification},
\newblock in: \bibinfo{booktitle}{Proceedings of the IEEE international
  conference on computer vision}, \bibinfo{year}{2015}, pp.
  \bibinfo{pages}{1026--1034}.
\bibitem[{Basdevant et~al.(1986)Basdevant, Deville, Haldenwang, Lacroix,
  Ouazzani, Peyret, Orlandi, and Patera}]{basdevant1986spectral}
\bibinfo{author}{C.~Basdevant}, \bibinfo{author}{M.~Deville},
  \bibinfo{author}{P.~Haldenwang}, \bibinfo{author}{J.~Lacroix},
  \bibinfo{author}{J.~Ouazzani}, \bibinfo{author}{R.~Peyret},
  \bibinfo{author}{P.~Orlandi}, \bibinfo{author}{A.~Patera},
\newblock \bibinfo{title}{Spectral and finite difference solutions of the
  burgers equation},
\newblock \bibinfo{journal}{Computers \& fluids} \bibinfo{volume}{14}
  (\bibinfo{year}{1986}) \bibinfo{pages}{23--41}.
\bibitem[{Sitzmann et~al.(2020)Sitzmann, Martel, Bergman, Lindell, and
  Wetzstein}]{sitzmann2020implicit}
\bibinfo{author}{V.~Sitzmann}, \bibinfo{author}{J.~Martel},
  \bibinfo{author}{A.~Bergman}, \bibinfo{author}{D.~Lindell},
  \bibinfo{author}{G.~Wetzstein},
\newblock \bibinfo{title}{Implicit neural representations with periodic
  activation functions},
\newblock \bibinfo{journal}{Advances in Neural Information Processing Systems}
  \bibinfo{volume}{33} (\bibinfo{year}{2020}) \bibinfo{pages}{7462--7473}.
\bibitem[{Tancik et~al.(2020)Tancik, Srinivasan, Mildenhall, Fridovich-Keil,
  Raghavan, Singhal, Ramamoorthi, Barron, and Ng}]{tancik2020fourier}
\bibinfo{author}{M.~Tancik}, \bibinfo{author}{P.~Srinivasan},
  \bibinfo{author}{B.~Mildenhall}, \bibinfo{author}{S.~Fridovich-Keil},
  \bibinfo{author}{N.~Raghavan}, \bibinfo{author}{U.~Singhal},
  \bibinfo{author}{R.~Ramamoorthi}, \bibinfo{author}{J.~Barron},
  \bibinfo{author}{R.~Ng},
\newblock \bibinfo{title}{Fourier features let networks learn high frequency
  functions in low dimensional domains},
\newblock \bibinfo{journal}{Advances in Neural Information Processing Systems}
  \bibinfo{volume}{33} (\bibinfo{year}{2020}) \bibinfo{pages}{7537--7547}.
\bibitem[{Wang et~al.(2021)Wang, Wang, and Perdikaris}]{wang2021eigenvector}
\bibinfo{author}{S.~Wang}, \bibinfo{author}{H.~Wang},
  \bibinfo{author}{P.~Perdikaris},
\newblock \bibinfo{title}{On the eigenvector bias of fourier feature networks:
  From regression to solving multi-scale pdes with physics-informed neural
  networks},
\newblock \bibinfo{journal}{Computer Methods in Applied Mechanics and
  Engineering} \bibinfo{volume}{384} (\bibinfo{year}{2021})
  \bibinfo{pages}{113938}.
\bibitem[{Wong et~al.(2022)Wong, Ooi, Gupta, and Ong}]{wong2022learning}
\bibinfo{author}{J.~C. Wong}, \bibinfo{author}{C.~Ooi},
  \bibinfo{author}{A.~Gupta}, \bibinfo{author}{Y.-S. Ong},
\newblock \bibinfo{title}{Learning in sinusoidal spaces with physics-informed
  neural networks},
\newblock \bibinfo{journal}{IEEE Transactions on Artificial Intelligence}
  (\bibinfo{year}{2022}).
\bibitem[{Evans(1998)}]{evans1998partial}
\bibinfo{author}{L.~C. Evans},
\newblock \bibinfo{title}{Partial differential equations},
\newblock \bibinfo{journal}{Graduate studies in mathematics}
  \bibinfo{volume}{19} (\bibinfo{year}{1998}) \bibinfo{pages}{7}.
\bibitem[{Grisvard(2011)}]{grisvard2011elliptic}
\bibinfo{author}{P.~Grisvard}, \bibinfo{title}{Elliptic problems in nonsmooth
  domains}, \bibinfo{publisher}{SIAM}, \bibinfo{year}{2011}.
\bibitem[{Benia and Sadallah(2016)}]{benia2016existence}
\bibinfo{author}{Y.~Benia}, \bibinfo{author}{B.-K. Sadallah},
\newblock \bibinfo{title}{Existence of solutions to burgers equations in
  domains that can be transformed into rectangles},
\newblock \bibinfo{journal}{Electronic Journal of Differential Equations}
  \bibinfo{volume}{2016} (\bibinfo{year}{2016}) \bibinfo{pages}{1--13}.
\bibitem[{Gagliardo(1957)}]{gagliardo1957caratterizzazioni}
\bibinfo{author}{E.~Gagliardo},
\newblock \bibinfo{title}{Caratterizzazioni delle tracce sulla frontiera
  relative ad alcune classi di funzioni in $ n $ variabili},
\newblock \bibinfo{journal}{Rendiconti del seminario matematico della
  universita di Padova} \bibinfo{volume}{27} (\bibinfo{year}{1957})
  \bibinfo{pages}{284--305}.

\end{thebibliography}





\end{document}